\documentclass[a4paper,UKenglish,cleveref,autoref,thm-restate]{lipics-v2021}
\pdfoutput=1

\title{Strategy Complexity of Mean Payoff, Total Payoff and Point Payoff Objectives in Countable MDPs}
\titlerunning{Strategy Complexity of Mean/Total/Point Payoff Objectives in Countable MDPs}
\author{Richard Mayr}{University of Edinburgh, UK}{}{}{}
\author{Eric Munday}{University of Edinburgh, UK}{}{}{}
\authorrunning{R.~Mayr and E.~Munday}
\Copyright{Richard Mayr and Eric Munday}

\relatedversion{This is the full version of a CONCUR 2021 paper \cite{MM:CONCUR2021}.}%optional, e.g. full version hosted on arXiv, HAL, or other respository/website
\supplement{}%optional, e.g. related research data, source code, ... hosted on a repository like zenodo, figshare, GitHub, ...
%mandatory
\supplement{}%optional, e.g. related research data, source code, ... hosted on a repository like zenodo, figshare, GitHub, ...

\funding{}%optional, to capture a funding statement, which applies to all authors. Please enter author specific funding statements as fifth argument of the \author macro.

\nolinenumbers %uncomment to disable line numbering

\hideLIPIcs  %uncomment to remove references to LIPIcs series (logo, DOI, ...), e.g. when preparing a pre-final version to be uploaded to arXiv or another public repository

\EventEditors{Serge Haddad and Daniele Varacca}
\EventNoEds{2}
\EventLongTitle{32nd International Conference on Concurrency Theory (CONCUR 2021)}
\EventShortTitle{CONCUR 2021}
\EventAcronym{CONCUR}
\EventYear{2021}
\EventDate{August 23--27, 2021}
\EventLocation{Virtual Conference}
\EventLogo{}
\SeriesVolume{203}
\ArticleNo{2}
\usepackage{amsmath,amsthm}
\usepackage{graphicx}
\usepackage{graphics}
\usepackage{latexsym}
\usepackage{amssymb}
\usepackage{stmaryrd}
\usepackage{ifthen}
\usepackage{array}
\usepackage{mathrsfs}
\usepackage{multirow}
\usepackage{bbm}

\usepackage{pgf}
\usepackage{pgfplots}
\pgfplotsset{compat=1.15}
\usepackage{tikz, mathtools, thm-restate, stmaryrd}
\usetikzlibrary{arrows,calc,automata,backgrounds,decorations,intersections,pgfplots.fillbetween,shapes.misc}
\tikzset{every picture/.style={thick,>=angle 60}}
\tikzset{MDPrand/.style={draw,circle,minimum size=11*1.5,inner sep=0}}
\tikzset{MDPcont/.style={draw,rectangle,minimum size=9*1.5,inner sep=0}}
\tikzset{MDPbad/.style={fill=red}}
\definecolor{myorange}{RGB}{255, 128, 0}

\usetikzlibrary{shapes.multipart,shapes.misc}
\usepgflibrary{decorations.pathreplacing}
\usetikzlibrary{arrows,calc,topaths}
\usetikzlibrary{automata,shapes.arrows, backgrounds, intersections, positioning}
\usetikzlibrary{shadows}

\usetikzlibrary{decorations.pathmorphing}
\usetikzlibrary{decorations.shapes}
\usetikzlibrary{shapes}
\usetikzlibrary{backgrounds}
\usetikzlibrary{fit}
\usetikzlibrary{arrows}
\usetikzlibrary{calc}
\usetikzlibrary{mindmap}
\usetikzlibrary{matrix}

\usepackage{mathtools}
\usepackage{thmtools,thm-restate}
\usepackage{xcolor}

\usepackage{bbding}

\usepackage{hyperref} % ,cleveref}
\hypersetup{colorlinks,linkcolor=blue,citecolor=blue,urlcolor=blue}
\newcommand{\+}[1]{\mathbb{#1}}

\newcommand{\N}{\+{N}}
\newcommand{\R}{\+{R}}
\newcommand{\x}{\times}

\usepackage{wasysym} % for "small" \ocircle symbol
\newcommand{\rsymbol}{\ocircle}
\newcommand{\zsymbol}{\Box}
\newcommand{\zstates}{\states_\zsymbol}
\newcommand{\rstates}{\states_\rsymbol}
\newcommand{\reachset}{T}

\newcommand{\eqby}[2][=]{\stackrel{\text{{\tiny{#2}}}}{#1}}
\newcommand{\eqdef}{\eqby{def}}
\newcommand{\defeq}{\eqdef}

\newcommand{\eps}{\varepsilon}

\newcommand{\problemx}[3]{
\par\noindent\underline{\sc#1}\par\nobreak\vskip.2\baselineskip
\begingroup\clubpenalty10000\widowpenalty10000
\setbox0\hbox{\bf INPUT:\ }\setbox1\hbox{\bf QUESTION:\ }
\dimen0=\wd0\ifnum\wd1>\dimen0\dimen0=\wd1\fi
\vskip-\parskip\noindent
\hbox to\dimen0{\box0\hfil}\hangindent\dimen0\hangafter1\ignorespaces#2\par
\vskip-\parskip\noindent
\hbox to\dimen0{\box1\hfil}\hangindent\dimen0\hangafter1\ignorespaces#3\par
\endgroup}

\newcommand{\Runs}[2]{{\textit{Runs}_{#1}^{#2}}}
\newcommand{\pRuns}[2]{{\textit{pRuns}_{#1}^{#2}}}
\newcommand{\dist}{\mathcal{D}}

\newcommand{\always}{{\sf G}}
\newcommand{\eventually}{{\sf F}}
\newcommand{\hide}[1]{}

\newcommand{\lrc}[1]{(#1)}

\newcommand{\ignore}[1]{}

\newcommand{\tuple}[1]{\lrc{#1}}

\newcommand{\denotationof}[2]{\llbracket #1\rrbracket^{#2}}

\newcommand{\mdp}{{\mathcal M}}

\newcommand{\mdptuple}{\tuple{\states,\zstates,\rstates,\transition,\probp,r}}

\newcommand{\states}{S}

\newcommand{\state}{s}

\newcommand{\transition}{{\longrightarrow}}

\newcommand{\probp}{P}

\newcommand{\complementof}[1]{\overline{#1}}

\newcommand{\play}{\rho}
\newcommand{\playset}{{\mathfrak R}}

\newcommand{\partialplay}{\rho}

\newcommand{\zstrat}{\sigma}

\newcommand{\zstratset}{\Sigma}

\newcommand{\memory}{{\sf M}}
\newcommand{\memconf}{{\sf m}}

\newcommand{\expectval}{{\mathcal E}}
\newcommand{\probm}{{\mathcal P}}

\newcommand{\formula}{{\varphi}}

\newcommand{\valueof}[2]{{\mathtt{val}_{#1}(#2)}}

\mathchardef\mhyphen="2D % define a "math hyphen"

\newcommand{\liminfmpobj}{{\it MP}_{\liminf\ge 0}}
\newcommand{\liminftpobj}{{\it TP}_{\liminf\ge 0}}
\newcommand{\liminfppobj}{{\it PP}_{\liminf\ge 0}}

\newcommand{\transience}{\mathtt{Transience}}
\newcommand{\reward}{\mathit{r}}

\begin{document}
\maketitle
\begin{abstract}
  We study countably infinite Markov decision processes (MDPs)
  with real-valued transition rewards.
  Every infinite run induces the following sequences of payoffs:
  1. Point payoff (the sequence of directly seen transition rewards),
  2. Total payoff (the sequence of the sums of all rewards so far), and
  3. Mean payoff.
  For each payoff type, the objective is to maximize the probability
  that the $\liminf$ is non-negative.
  We establish the complete picture of the strategy complexity of
  these objectives, i.e., how much memory is necessary and sufficient for
  $\eps$-optimal (resp.\ optimal) strategies.
  Some cases can be won with memoryless deterministic strategies,
  while others require a step counter, a reward counter, or both.
\end{abstract}

\section{Introduction}\label{sec:intro}
{\bf\noindent Background.}
Markov decision processes (MDPs) are a standard model for dynamic systems that
exhibit both stochastic and controlled behavior \cite{Puterman:book}.
Applications include control theory~\cite{blondel2000survey,NIPS2004_2569},
operations research
and finance~\cite{Flesch:JOTA2020,bauerle2011finance,schal2002markov},
artificial intelligence and machine
learning~\cite{sutton2018reinforcement,sigaud2013markov},
and formal verification~\cite{ModCheckHB18,ModCheckPrinciples08}.

An MDP is a directed graph where states are either random or controlled.
In a random state the next state is chosen according to a fixed probability distribution.
In a controlled state the controller can choose
a distribution over all possible successor states.
By fixing a strategy for the controller (and an initial state), one obtains a probability space
of runs of the MDP. The goal of the controller is to optimize the expected value of
some objective function on the runs.
The type of strategy necessary to achieve an $\eps$-optimal (resp.\ optimal)
value for a given objective is called its \emph{strategy complexity}.

{\bf\noindent Transition rewards and liminf objectives.}
MDPs are given a reward structure by assigning a real-valued
(resp.\ integer or rational) reward to
each transition. Every run then induces an infinite sequence of
seen transition rewards $r_0r_1r_2\dots$.
We consider the $\liminf$
of this sequence, as well as
two other important derived sequences.
\begin{description}
\item[1.]
The point payoff considers the
$\liminf$ of the sequence $r_0 r_1r_2 \dots$
directly.
\item[2.]
The total payoff considers the
$\liminf$ of the sequence
$\left\{\sum_{i=0}^{n-1} r_i\right\}_{n \in \N}$, i.e., the sum of all rewards seen
so far.
\item[3.]
The mean payoff considers the
$\liminf$ of the sequence
$\left\{\frac{1}{n}\sum_{i=0}^{n-1} r_i\right\}_{n \in \N}$, i.e., the mean of
all rewards seen
so far in an expanding prefix of the run.
\end{description}
For each of the three cases above, the $\liminf$ threshold objective is
to maximize the probability that the $\liminf$ of the respective type of sequence
is $\ge 0$.

{\bf\noindent Our contribution.}
We establish the strategy complexity of all the
$\liminf$ threshold objectives above for \emph{countably infinite} MDPs.
(For the simpler case of finite MDPs, see the paragraph on related work below.)
We show the amount and type of memory
that is sufficient for $\eps$-optimal strategies
(and optimal strategies, where they exist),
and corresponding lower bounds in the sense of \cref{rem:lowerbonds}. 
This is not only the distinction between memoryless, finite memory
and infinite memory, but the type of infinite memory that is necessary
and sufficient.
A step counter is an integer counter that merely counts the number of steps in
the run (i.e., like a discrete clock), while a reward counter is a variable that records the
sum of all rewards seen so far.
(The reward counter has the same type as the transition rewards in the MDP, i.e.,
integers, rationals or reals.)
While these use infinite memory, it is a very
restricted form, since this memory is not directly controlled by the player.
Strategies using only a step counter are also called Markov strategies
\cite{Puterman:book}.

Some of the $\liminf$ objectives can be attained by memoryless deterministic (MD)
strategies, while others require (in the sense of \cref{rem:lowerbonds})
a step counter, 
a reward counter, or both. It depends on the type of objective (point, total, or
mean payoff) and on whether the MDP is finitely or infinitely
branching.
For clarity of presentation,
our counterexamples use large transition rewards and high degrees of
branching. However, the lower bounds hold even for just binary
branching MDPs with transition rewards in $\{-1,0,1\}$;
cf.~\cref{app:strengthening}.

For our objectives, the strategy complexities  of $\eps$-optimal
and optimal strategies (where they exist) coincide, but the proofs are different.
\cref{table:allresults} shows the results for all combinations.

\begin{table*}[hbtp]
\centering
\begin{tabular}{|l||c|c|c|}\hline
      & Point payoff & Total payoff & Mean payoff \\ \hline
$\eps$-optimal, infinitely branching & SC \ref{infbranchsteplower}, \ref{infpointpayoff}  & SC+RC \ref{infbranchsteplower}, \ref{infinitesummarytp}, \ref{inftpepsupper} & SC+RC \ref{mpstepepslower}, \ref{infinitesummary}, \ref{mpepsupper}  \\ \hline
optimal, infinitely branching     & SC \ref{infbranchsteplower}, \ref{infoptupper}  & SC+RC \ref{almostsummarytp}, \ref{infbranchsteplower}, \ref{infoptupper} & SC+RC \ref{almostsummary}, \ref{mpstepoptlower}, \ref{infoptupper}   \\ \hline
$\eps$-optimal, finitely branching   & MD \ref{finpointpayoff} & RC \ref{infinitesummarytp}, \ref{fintpepsupper}   & SC+RC \ref{mpstepepslower}, \ref{infinitesummary}, \ref{mpepsupper}  \\ \hline
optimal, finitely branching       & MD \ref{finoptupper} & RC \ref{almostsummarytp}, \ref{finoptupper}   & SC+RC \ref{almostsummary}, \ref{mpstepoptlower}, \ref{infoptupper} \\ \hline
\end{tabular}
\caption{Strategy complexity of $\eps$-optimal/optimal
strategies for point, total and mean payoff objectives in
infinitely/finitely branching MDPs. MD stands for memoryless deterministic,
SC for step counter, RC for reward counter and SC+RC for both.
All strategies are deterministic and randomization does not help.
For each result, we list the numbers of the theorems that show the
upper and lower bounds on the strategy complexity.
The lower bounds hold in the sense of \cref{rem:lowerbonds},
but work for integer
% transition
rewards. The upper
bounds hold even for real-valued rewards. 
}\label{table:allresults}
\end{table*}

\vspace*{-5mm}
Some complex new proof techniques are developed to show these results.
E.g., the examples showing the lower bound in cases where both a step counter
and a reward counter are required use a finely tuned tradeoff between different
risks that can be managed with both counters, but not with just one
counter plus arbitrary finite memory.
The strategies showing the upper bounds need to take into account
convergence effects, e.g., the sequence of point rewards
$-1/2, -1/3, -1/4, \dots$ \emph{does} satisfy $\liminf$ $\ge 0$, i.e.,
one cannot assume that rewards are integers.

Due to space constraints, we sketch some proofs in the main body.
Full proofs can be found in the Appendix.

{\bf\noindent Related work.}
Mean payoff objectives for \emph{finite} MDPs have been widely
studied; cf.~survey in \cite{chatterjee2011games}.
There exist optimal MD strategies for $\liminf$ mean payoff
(which are also optimal for $\limsup$ mean payoff since the transition rewards are bounded),
and the associated computational problems
% (thresholds, winning, etc.)
can be
solved in polynomial time \cite{chatterjee2011games,Puterman:book}.
Similarly, see \cite{CDH:ICALP2009} for a survey on
$\limsup$ and $\liminf$ point payoff objectives in finite stochastic games and MDPs,
where there also exist optimal MD strategies, and the more recent paper
by Flesch, Predtetchinski and Sudderth \cite{FPS:2018} on simplifying optimal strategies.

All this does \emph{not} carry over to countably infinite MDPs.
Optimal strategies need not exist
(not even for much simpler objectives),
($\eps$-)optimal strategies can require
infinite memory, and computational problems are not defined in general, since
a countable MDP need not be finitely presented
\cite{KMSW2017}.
Moreover, attainment for $\liminf$ mean payoff need not coincide with attainment
for $\limsup$ mean payoff, even for very simple examples.
E.g., consider the acyclic infinite graph with
transitions $s_n \rightarrow s_{n+1}$ for all $n \in \N$
with reward $(-1)^n2^n$ in the $n$-th step, which yields a $\liminf$ mean payoff
of $-\infty$ and a $\limsup$ mean payoff of $+\infty$.

Mean payoff objectives for countably infinite MDPs have been considered in
\cite[Section 8.10]{Puterman:book}, e.g., 
\cite[Example 8.10.2]{Puterman:book}
(adapted in \cref{putermanexample})
shows that there are no optimal MD (memoryless deterministic)
strategies for $\liminf$/$\limsup$ mean payoff.
\cite[Counterexample 1.3]{Ross:1983} shows that there are not even
$\eps$-optimal memoryless randomized strategies for $\liminf$/$\limsup$ mean payoff.
(We show much stronger lower/upper bounds; cf.~\cref{table:allresults}.)

Sudderth \cite{Sudderth:2020} considered an objective on countable MDPs that is
related to our point payoff threshold objective. However, instead of maximizing the
probability that the $\liminf$/$\limsup$ is non-negative, it asks to maximize the
\emph{expectation} of the $\liminf$/$\limsup$ point payoffs, which is a different
problem (e.g., it can tolerate a high probability of a negative $\liminf$/$\limsup$
if the remaining cases have a huge positive $\liminf$/$\limsup$).
Hill \& Pestien \cite{Hill-Pestien:1987} showed the existence of good
randomized Markov strategies for the $\limsup$ of the \emph{expected}
average reward up-to step $n$ for growing $n$, and for the \emph{expected}
$\liminf$ of the point payoffs.

\section{Preliminaries}\label{sec:prelim}
\noindent{\bf Markov decision processes.}
A \emph{probability distribution} over a countable set $S$ is a function
$f:\states\to[0,1]$ with $\sum_{\state\in \states}f(\state)=1$.
We write 
%$\supp(f) \eqdef \{\state \in \states \mid f(\state) > 0\}$ for the \emph{support} of $f$.
%Let 
$\dist(\states)$ for the set of all probability distributions over $\states$. 
A \emph{Markov decision process} (MDP) $\mdp=\mdptuple$ consists of
a countable set~$\states$ of \emph{states}, 
which is partitioned into a set~$\zstates$ of \emph{controlled states} 
and  a set~$\rstates$ of \emph{random states},
a  \emph{transition relation} $\transition\subseteq\states\x\states$,
and a  \emph{probability function}~$\probp:\rstates \to \dist(\states)$. 
We  write $\state\transition{}\state'$ if $\tuple{\state,\state'}\in \transition$,
and  refer to~$s'$ as a \emph{successor} of~$s$. 
We assume that every state has at least one successor.  
The probability function~$P$  assigns to each random state~$\state\in \rstates$
a probability distribution~$P(\state)$ over its (non-empty) set of successor states.
A \emph{sink in $\mdp$} is a subset $T \subseteq \states$ closed under the $\transition$ relation,
that is,  $\state \in \reachset$ and  $\state\transition\state'$ implies that $\state'\in T$.

An MDP is \emph{acyclic} if the underlying directed graph~$(S,\transition)$ is acyclic, i.e., 
there is no directed cycle.
It is  \emph{finitely branching} 
if every state has finitely many successors
and \emph{infinitely branching} otherwise.
An MDP without controlled states
($\zstates=\emptyset$) is called a \emph{Markov chain}.

In order to specify our mean/total/point payoff objectives (see below), we 
define a function $\reward: \states \times \states \to \R$ that assigns numeric
rewards to transitions.

\smallskip
\noindent{\bf Strategies and Probability Measures.}
A \emph{run}~$\play$ is an  infinite sequence of states and transitions
$\state_0e_0\state_1e_1\cdots$ 
such that $e_i = (\state_i, \state_{i+1}) \in \transition$
for all~$i\in \mathbb{N}$.
Let $\Runs{\mdp}{\state_0}$
be the set of all runs from $\state_0$ in the MDP $\mdp$.
A \emph{partial run} is a finite prefix of a run, 
$\pRuns{\mdp}{\state_0}$ is the set of all partial runs
from $\state_0$ and $\pRuns{\mdp}{}$ the set of partial runs from any state.

We write~$\play_s(i)\eqdef\state_i$ for the $i$-th state along~$\play$
and $\play_e(i)\eqdef e_i$ for the $i$-th transition along~$\play$.
We sometimes write runs as $\state_0\state_1\cdots$, leaving the transitions implicit.
We say that a (partial) run $\play$ \emph{visits} $\state$ if
$\state=\play_s(i)$ for some $i$, and that~$\play$ starts in~$s$ if $\state=\play_s(0)$. 

A \emph{strategy} 
is a function $\zstrat:\pRuns{\mdp}{}\!\cdot\!\zstates \to \dist(S)$ that 
assigns to partial runs $\partialplay\state$, where $\state \in \zstates$,
a distribution over the successors~$\{\state'\in \states\mid \state \transition{} \state'\}$. 
The set of all strategies  in $\mdp$ is denoted by $\zstratset_\mdp$ 
(we omit the subscript and write~$\zstratset$ if $\mdp$ is clear from the context).
A (partial) run~$\state_0e_0\state_1e_1\cdots$ is consistent with a strategy~$\zstrat$
if for all~$i$
either $\state_i \in \zstates$ and $\zstrat(\state_0e_0\state_1e_1\cdots\state_i)(\state_{i+1})>0$,
or
$\state_i \in \rstates$ and $\probp(\state_i)(\state_{i+1})>0$.
%if~$\state_{i+1}\in \supp(\zstrat(\state_0\state_1\cdots\state_i))$
%for all~$i$ with~$\state_i \in \zstates$, and
%$\state_{i+1}\in \supp(\probp(\state_i))$ for all~$i$ with~$\state_i \in \rstates$.
%
%We use~$\playsof{\mdp,\state,\zstrat}$ to denote the set of all runs in $\mdp$ starting in state $\state$ induced by $\zstrat$.

An MDP $\mdp=\mdptuple$, an initial state $\state_0\in \states$, and a strategy~$\zstrat$ 
induce a probability space in which the outcomes are runs starting in $\state_0$
and with measure $\probm_{\mdp,\state_0,\zstrat}$
%$\playset \subseteq \state_0 \states^\omega$
%of runs starting from $\state_0$
defined as follows.
%We write $\probm_{\mdp,\state_0,\zstrat}({\playset})$ for the probability of a 
%measurable set $\playset \subseteq \state_0 \states^\omega$ of runs starting from~$\state_0$.
It is first defined on \emph{cylinders} $s_0 e_0 s_1 e_1 \ldots
s_n \Runs{\mdp}{s_n}$:
if $s_0 e_0 s_1 e_1 \ldots s_n$
is not a partial run consistent with~$\zstrat$ then
$\probm_{\mdp,\state_0,\zstrat}(s_0 e_0 s_1 e_1 \ldots s_n \Runs{\mdp}{s_n}) \eqdef 0$.
Otherwise, $\probm_{\mdp,\state_0,\zstrat}(s_0 e_0 s_1 e_1 \ldots
s_n \Runs{\mdp}{s_n}) \eqdef \prod_{i=0}^{n-1} \bar{\zstrat}(s_0 e_0
s_1 \ldots s_i)(s_{i+1})$, where $\bar{\zstrat}$ is the map that
extends~$\zstrat$ by $\bar{\zstrat}(w s) = \probp(s)$
for all partial runs $w s \in \pRuns{\mdp}{}\!\cdot\!\rstates$.
By Carath\'eodory's theorem~\cite{billingsley-1995-probability}, 
this extends uniquely to a probability measure~$\probm_{\mdp,\state_0,\zstrat}$ on 
the Borel $\sigma$-algebra $\?F$ of subsets of~$\Runs{\mdp}{s_0}$.
Elements of $\?F$, i.e., measurable sets of runs, are called \emph{events} or \emph{objectives} here.
For $X\in\?F$ we will write $\complementof{X}\eqdef \Runs{\mdp}{s_0}\setminus X\in \?F$ for its complement
and $\expectval_{\mdp,\state_0,\zstrat}$
for the expectation wrt.~$\probm_{\mdp,\state_0,\zstrat}$.
We drop the indices if possible without
% introducing
ambiguity.

\noindent{\bf Objectives.}
We consider objectives that are determined by a predicate on infinite runs.
We assume familiarity with the syntax and semantics of the temporal
logic LTL \cite{CGP:book}.
Formulas are interpreted on the structure $(\states,\transition)$.
We use $\denotationof{\formula}{\state}$ to denote the set of runs starting from
$\state$ that satisfy the LTL formula $\formula$,
which is a measurable set \cite{Vardi:probabilistic}.
We also write $\denotationof{\formula}{}$ for $\bigcup_{s \in S} \denotationof{\formula}{s}$.
Where it does not cause confusion we will
identify $\varphi$ and $\denotationof{\formula}{}$
and just write
$\probm_{\mdp,\state,\zstrat}(\formula)$
instead of 
$\probm_{\mdp,\state,\zstrat}(\denotationof\formula\state)$.
The reachability objective of eventually visiting a set of states $X$ can be expressed by
$\denotationof{\eventually X}{} \eqdef \{\play\,|\, \exists
i.\, \play_s(i) \in X\}$.
Reaching $X$ within at most $k$ steps is expressed by
$\denotationof{\eventually^{\le k} X}{} \eqdef \{\play\,|\, \exists
i \le k.\, \play_s(i) \in X\}$.
The definitions for eventually visiting certain transitions are analogous.
The operator $\always$ (always) is defined as $\neg\eventually\neg$.
So the safety objective of avoiding $X$ is expressed by
$\always \neg X$.

% We consider the following objectives.
\begin{itemize}
\item
The $\liminfppobj$ objective is to maximize the
probability that the $\liminf$ of the \emph{point} payoffs (the immediate
transition rewards) is $\ge 0$, i.e., 
$\liminfppobj \defeq
\{\rho \mid \liminf_{n\in \N} \reward(\rho_e(n)) \ge 0\}$.
\item
The $\liminftpobj$ objective is to maximize the
probability that the $\liminf$ of the \emph{total} payoff (the sum of the
transition rewards seen so far) is $\ge 0$, i.e., 
$\liminftpobj \defeq
\{\rho \mid \liminf_{n\in \N} \sum_{j=0}^{n-1}\reward(\rho_e(j)) \ge 0\}$.
\item
The $\liminfmpobj$ objective is to maximize the
probability that the $\liminf$ of the \emph{mean} payoff
is $\ge 0$, i.e., 
$\liminfmpobj\defeq
\{\rho \mid \liminf_{n\in \N}\frac{1}{n}\sum_{j=0}^{n-1}
\reward(\rho_e(j)) \ge 0\}$.
\end{itemize}
An objective $\formula$ is called \emph{tail} in $\mdp$
if for every run
$\rho'\rho$ in $\mdp$ with some finite prefix $\rho'$ we have
$\rho'\rho \in \denotationof{\formula}{} \Leftrightarrow \rho \in \denotationof{\formula}{}$.
An objective is called a \emph{tail objective} if it is tail in every MDP. 
%Note that
$\liminfppobj$ and $\liminfmpobj$ are tail objectives,
but $\liminftpobj$ is not.
Also $\liminfppobj$ is more general than co-B\"uchi.
(The special case of integer transition rewards coincides with co-B\"uchi,
since rewards $\le -1$ and accepting states can be encoded into each other.)

\smallskip
\noindent{\bf  Strategy Classes.}
Strategies are in general  \emph{randomized} (R) in the sense that they take values in $\dist(\states)$. 
A strategy~$\zstrat$ is \emph{deterministic} (D) if $\zstrat(\rho)$ is a Dirac
distribution for all $\rho$.
General strategies can be \emph{history dependent} (H), while others are
restricted by the size or type of memory they use, see below.
We consider certain classes of strategies:
\begin{itemize}
\item
A strategy $\zstrat$ is  \emph{memoryless}~(M) (also called \emph{positional}) 
if it can be implemented with a memory  of size~$1$. 
%In other words, if $\zstrat(\rho \state)=\zstrat(\rho' \state)$
%for all $\state\in \zstates$ and $\rho,\rho'\in S^*$).
We may view
M-strategies as functions $\zstrat: \zstates \to \dist(\states)$.
% \item A strategy $\zstrat$ is \emph{$k$-bit} 
% if it can be implemented with $\le k$ bits of memory, i.e., $\le 2^k$ memory modes.
% Such a strategy is then determined by a function
% $\tau:\{0,1,\dots,2^k-1\}\times \states \to \dist(\{0,1,\dots,2^k-1\} \times \states)$.
\item %The  memory required by~$\zstrat$ is
%the cardinality  of the smallest set  $\memory$ that can implement $\zstrat$.
%In particular, 
%a strategy~$\zstrat$ is \emph{finite memory}~(F) if 
%there exists a finite~$\memory$ implementing~$\zstrat$.
%otherwise we say that~$\zstrat$ \emph{requires infinite memory}.
A strategy~$\zstrat$ is \emph{finite memory}~(F) if 
there exists a finite memory~$\memory$ implementing~$\zstrat$.
Hence FR stands for finite memory randomized.
\item
% A step counter strategy uses infinite memory, but only in a very restricted
% way. 
A \emph{step counter strategy} bases decisions only on
the current state and the number of steps taken so far, i.e., it uses
an unbounded integer counter that gets incremented by $1$ in every step.
% It can be implemented with the natural numbers $\mathbb{N}$
% as the memory, and a function $\tau$ such that the distribution
% $\tau(\memconf,\state)$ is over $\{\memconf+1\}\times \states$
% for all $\memconf\in \memory$ and $\state\in \states$.
Such strategies are also called \emph{Markov strategies} \cite{Puterman:book}.
\item
\emph{$k$-bit Markov strategies} use $k$ extra bits of general purpose memory
in addition to a step counter \cite{KMST2020c}.
\item
A \emph{reward counter strategy} uses infinite memory, but only in
the form of a counter that always contains the sum of all transition rewards
seen to far.
\item
A \emph{step counter + reward counter strategy} uses both a step counter and a reward counter.
\end{itemize}
See \cref{app-def} for a formal definition how strategies use memory.
Step counters and reward counters are very restricted forms of
memory, since the memory update is not directly under the control of the
player. These counters merely record an aspect of the partial run.

\smallskip
\noindent{\bf Optimal and $\eps$-optimal Strategies.}
Given an objective~$\formula$, the value of state~$s$ in an MDP~$\mdp$, denoted by 
$\valueof{\mdp,\formula}{s}$, is the supremum probability of achieving~$\formula$.
 Formally, $\valueof{\mdp,\formula}{s} \eqdef\sup_{\sigma \in \Sigma} \probm_{\mdp,\state,\zstrat}(\formula)$ where $\Sigma$ is the set of all strategies.
For $\eps\ge 0$ and state~$s\in\states$, we say that a strategy is \emph{$\eps$-optimal} from $s$
if $\probm_{\mdp,\state,\zstrat}(\formula) \geq \valueof{\mdp,\formula}{s} -\eps$.
A $0$-optimal strategy is called \emph{optimal}. 
An optimal strategy is \emph{almost-surely winning} if $\valueof{\mdp,\formula}{s} = 1$. 
Considering an MD strategy as a function $\zstrat: \zstates \to \states$ and $\eps\ge 0$, $\zstrat$ is \emph{uniformly} $\eps$-optimal  (resp.~uniformly optimal) if it is $\eps$-optimal (resp.~optimal) from \emph{every} $s\in S$.

\begin{remark}\label{rem:lowerbonds}
\rm
To establish an upper bound $X$ on the strategy complexity of an objective
$\formula$ in countable MDPs,
it suffices to prove that there always exist good
($\eps$-optimal, resp.\ optimal)
strategies in class $X$
(e.g., MD, MR, FD, FR, etc.)
for objective $\formula$.

Lower bounds on the strategy complexity of an objective $\formula$
can only be established in the sense of proving that good
strategies for $\formula$ do not exist in some classes $Y$,
$Z$, etc.
% Different classes of finite-memory strategies are comparable by comparing the
% number of allowed memory modes, e.g., strategies using 10 memory modes can do the same
% or more than strategies using just 2 memory modes.
Classes of strategies that use different types
of \emph{restricted} infinite memory are generally not comparable,
e.g., step counter strategies are
incomparable to reward counter strategies.
In particular, there is no weakest type of infinite memory with restricted use.
Therefore statements like ``good strategies for objective $\formula$ require at
least a step counter'' are always \emph{relative} to the
considered alternative strategy classes.
% They mean that good strategies for $\formula$ do not exist
% in any of these alternative classes, none of which allows a step counter
% (or something stronger).
In this paper, we only consider the strategy classes of memoryless, finite memory,
step counter, reward counter and \emph{combinations thereof}.
Thus, when we write in \cref{table:allresults} that an objective requires a
step counter (SC), it just means that a reward counter (RC) plus finite memory
is not sufficient.
\end{remark}
For our upper bounds, we use deterministic strategies.
Moreover, we show that allowing randomization does not help
to reduce the strategy complexity, in the sense of \cref{rem:lowerbonds}.

\section{When is a step counter not sufficient?}\label{sec:liminfreward}
In this section we will prove that strategies with a step counter plus
arbitrary finite memory are not sufficient
for $\varepsilon$-optimal strategies for $\liminfmpobj$ or $\liminftpobj$.
We will construct an acyclic MDP where the step counter is implicit in the
state such that $\varepsilon$-optimal
strategies for $\liminfmpobj$ and $\liminftpobj$
still require infinite memory.

\subsection{Epsilon-optimal strategies}

We construct an acyclic MDP $\mathcal{M}$ in which the step counter is implicit in the state as follows.

%%%%%%%%%%%%%%%%%%%%%%%%%%%%%%%%%%%%%%%%%%%%%%%%%%%%%%

%%%%%%%%%%%%%%%%%%%%%%%%%%

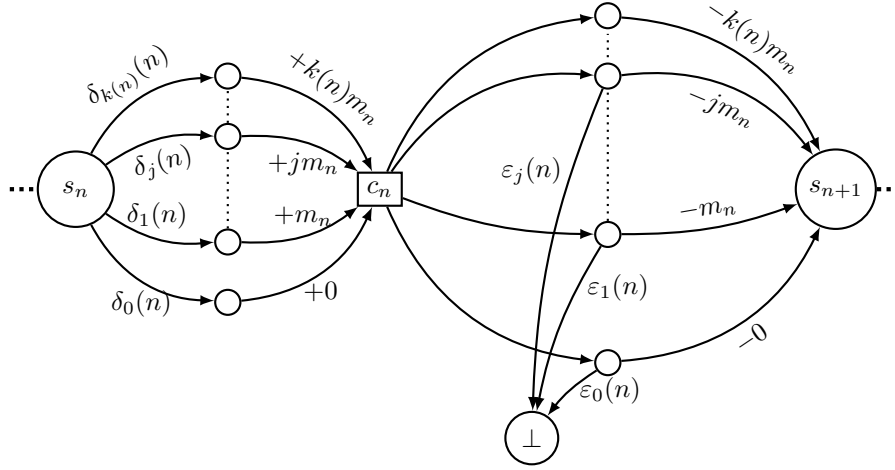
\begin{figure*}
\begin{center}
    \begin{tikzpicture}

    \node[draw,circle, minimum height=1cm] (R) at (0,0) {$s_{n}$};
    \node[draw, minimum height=0.4cm, minimum width=0.4cm] (N) at (4,0) {$c_{n}$};
    \node[draw,circle] (S) at (10,0) {$s_{n+1}$};
    
    %Dots off to the left and right
    \node[] (Left) at (-1,0) {};
    \node[] (Right) at (11,0) {};
    \draw [dotted, ultra thick] (Left) -- (R);
    \draw [dotted, ultra thick] (S) -- (Right);

    \node[draw,circle] (M) at (7,-0.6){};
    \node[draw,circle] (U) at (7,1.5){};
    
    \node[draw, circle] at (7,-2.3) (B) {};
    \node[draw, circle] at (7,2.3) (E) {};
    
    \node[draw, circle] at (6,-3.3) (Deadbottom) {$\perp$};
    %\node[draw, circle] at (7,-1.6) (Deadmiddle) {$\perp$};
    %\node[draw, circle] at (7,0.5) (Deadtop) {$\perp$};
    
    %Left hand side dotted lines
    % 
    % 
    % 
    % 
    % 

    \node[draw,circle] (MidBotL) at (2,-0.7){};
    \node[draw,circle] (MidTopL) at (2,0.7){};
    
    \node[draw, circle] (BotL) at (2,-1.5) {};
    \node[draw, circle] (TopL) at (2,1.5) {};

    \draw [dotted, thick] (MidBotL) -- (MidTopL);
    \draw [dotted, thick] (MidTopL) -- (TopL);
    
    %Right hand side dotted lines
    % 
    % 
    % 
    % 
    % Invisible nodes for the shift
    \coordinate[shift={(0mm,2.5mm)}] (Mshift) at (M);
    %\coordinate[shift={(0mm,-4mm)}] (Deadtopshift) at (Deadtop);
    
    %\draw [dotted, thick] (Mshift) -- (Deadtopshift);

    \coordinate[shift={(0mm,2.5mm)}] (Ushift) at (U);
    \coordinate[shift={(0mm,-2.4mm)}] (Eshift) at (E);

    \draw [dotted, thick] (Ushift) -- (Eshift);

    %Draw the edges on the left
    
    \draw[->,>=latex] (R) edge[bend left] node[sloped, above, midway]{$\delta_{k(n)}(n)$} (TopL) 
                    (TopL) edge[bend left] node[above, sloped, midway]{$+k(n)m_{n}$} (N)
                    (R) edge[bend left=20] node[sloped, below, midway]{$\delta_{j}(n)$} (MidTopL)
                    (MidTopL) edge[bend left=20] node[below, midway]{$+jm_{n}$} (N)
                    (R) edge[bend right=20] node[above, midway]{$\delta_{1}(n) $} (MidBotL)
                    (MidBotL) edge[bend right=20] node[above, midway]{$+m_{n}$} (N)
                    (R) edge[bend right] node[below, midway]{$\delta_{0}(n)$} (BotL) 
                    (BotL) edge [bend right] node[below, midway]{$+0$} (N);

    %Draw the edges on the right

    \draw[->,>=latex] (N) edge[bend right=10] (M)
                      (N) edge[bend left]  (E)
                      (E) edge[bend left] node[sloped, above, midway]{$-k(n)m_{n}$} (S)                  
                      (N) edge[bend right] (B)
                      (B) edge[bend right] node[sloped, below, midway]{$-0$} (S)
                      (B) edge[bend right=10] node[right, midway]{$\varepsilon_{0}(n)$} (Deadbottom)
                      (M) edge[bend right=10] node[right, near start]{$\varepsilon_{1}(n)$} (Deadbottom)
                      (M) edge[bend right=10] node[sloped, above, midway]{$-m_{n}$} (S)
                      (N) edge[bend left] (U)
                      (U) edge[bend left] node[sloped, below, midway]{$-jm_{n}$} (S)
                      (U) edge[bend right=10] node[left, near start]{$\varepsilon_{j}(n)$} (Deadbottom);
    
    \draw [dotted, thick] (M) -- (U);

    \end{tikzpicture}
    \caption{A typical building block with $k(n)+1$ choices, first random then controlled. The number of choices $k(n) + 1$ grows unboundedly with $n$. This is the $n$-th building block of the MDP in \cref{chain}.
    The $\delta_{i}(n)$ and $\eps_{i}(n)$ are probabilities depending on $n$ and the $\pm i m_{n}$ are transition rewards. 
    We index the successor states of $s_{n}$ and $c_{n}$ from $0$ to $k(n)$ to
    match the indexing of the $\delta$'s and $\eps$'s such that the bottom
    state is indexed with $0$ and the top state with $k(n)$.
    }
    \label{infinitegadget}
\end{center}
    \end{figure*}

%%%%%%%%%%%%%%%%%%%%%%%%%

%%%%%%%%%%%%%%%%%%%%%%%%%%%%%%%%%%%%%%%%%%%%%%%%%%%%%%%

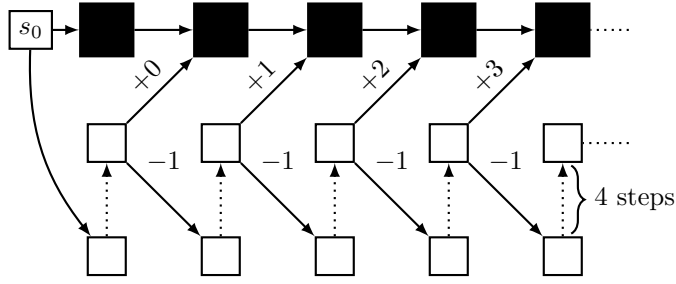
\begin{figure}[t]
\begin{center}
    \begin{tikzpicture}
    
    % invisible nodes
    \node[draw, minimum height=0.5cm, minimum width=0.5cm] (I1) at (-0.5,0) {$s_{0}$};
    \node (I2) at (7.5,0) {};
    \node (I3) at (7.5,-1.5) {};
    
    % the gadget nodes
    \node[draw, minimum height=0.7cm, minimum width=0.7cm, fill=black] (B1) at (0.5,0) {B1};
    \node[draw, minimum height=0.7cm, minimum width=0.7cm, fill=black] (B2) at (2,0) {B2};
    \node[draw, minimum height=0.7cm, minimum width=0.7cm, fill=black] (B3) at (3.5,0) {B3};
    \node[draw, minimum height=0.7cm, minimum width=0.7cm, fill=black] (B4) at (5,0) {B4};
    \node[draw, minimum height=0.7cm, minimum width=0.7cm, fill=black] (B5) at (6.5,0) {B5};
    
    \node[draw, minimum height=0.5cm, minimum width=0.5cm] (W1) at (0.5,-1.5) {};
    \node[draw, minimum height=0.5cm, minimum width=0.5cm] (W2) at (2,-1.5) {};
    \node[draw, minimum height=0.5cm, minimum width=0.5cm] (W3) at (3.5,-1.5) {};
    \node[draw, minimum height=0.5cm, minimum width=0.5cm] (W4) at (5,-1.5) {};
    \node[draw, minimum height=0.5cm, minimum width=0.5cm] (W5) at (6.5,-1.5) {};
    
    \node[draw, minimum height=0.5cm, minimum width=0.5cm] (W6) at (0.5,-3) {};
    \node[draw, minimum height=0.5cm, minimum width=0.5cm] (W7) at (2,-3) {};
    \node[draw, minimum height=0.5cm, minimum width=0.5cm] (W8) at (3.5,-3) {};
    \node[draw, minimum height=0.5cm, minimum width=0.5cm] (W9) at (5,-3) {};
    \node[draw, minimum height=0.5cm, minimum width=0.5cm] (W10) at (6.5,-3) {};
    
    %\node[draw, minimum height=0.5cm, minimum width=0.5cm] (W11) at (1,-4.5) {};
    %\node[draw, minimum height=0.5cm, minimum width=0.5cm] (W12) at (3,-4.5) {};
    %\node[draw, minimum height=0.5cm, minimum width=0.5cm] (W13) at (5,-4.5) {};
    %\node[draw, minimum height=0.5cm, minimum width=0.5cm] (W14) at (7,-4.5) {};
    %\node[draw, minimum height=0.5cm, minimum width=0.5cm] (W15) at (9,-4.5) {};
    
    %\node[draw, minimum height=0.5cm, minimum width=0.5cm] (W16) at (1,-6) {};
    %\node[draw, minimum height=0.5cm, minimum width=0.5cm] (W17) at (3,-6) {};
    %\node[draw, minimum height=0.5cm, minimum width=0.5cm] (W18) at (5,-6) {};
    %\node[draw, minimum height=0.5cm, minimum width=0.5cm] (W19) at (7,-6) {};
    %\node[draw, minimum height=0.5cm, minimum width=0.5cm] (W20) at (9,-6) {};

    % drawing the lines
    
    \draw[->, >=latex] (I1) -- (B1);
    \draw[dotted, thick] (I2) -- (B5);
    \draw[dotted, thick] (I3) -- (W5);
    \draw[->, >=latex] (I1) to[bend right=20] (W6);
    
    \draw[->,>=latex] (W1) -- (B2) node[sloped, above, midway]{$+0$};
    
    \draw[->,>=latex] (W2) -- (B3) node[sloped, above, midway]{$+1$};
    
    \draw[->,>=latex] (W3) -- (B4) node[sloped, above, midway]{$+2$};
    
    \draw[->,>=latex] (W4) -- (B5) node[sloped, above, midway]{$+3$};

    \draw[->,>=latex] (B1) -- (B2);
    \draw[->,>=latex] (B2) -- (B3);
    \draw[->,>=latex] (B3) -- (B4);
    \draw[->,>=latex] (B4) -- (B5);
    
    \draw[->, >=latex, dotted, thick] (W6) -- (W1);
    \draw[->, >=latex, dotted, thick] (W7) -- (W2);
    \draw[->, >=latex, dotted, thick] (W8) -- (W3);
    \draw[->, >=latex, dotted, thick] (W9) -- (W4);
    \draw[->, >=latex, dotted, thick] (W10) -- (W5);
    
    %\draw[->,>=latex] (W11) -- (W6);
    %\draw[->,>=latex] (W12) -- (W7);
    %\draw[->,>=latex] (W13) -- (W8);
    %\draw[->,>=latex] (W14) -- (W9);
    %\draw[->,>=latex] (W15) -- (W10);
    
    %\draw[->,>=latex] (W16) -- (W11);
    %\draw[->,>=latex] (W17) -- (W12);
    %\draw[->,>=latex] (W18) -- (W13);
    %\draw[->,>=latex] (W19) -- (W14);
    %\draw[->,>=latex] (W20) -- (W15);
    
    \draw[->,>=latex] (W1) -- (W7) node[above=0.25cm, midway]{$-1$};
    \draw[->,>=latex] (W2) -- (W8) node[above=0.25cm, midway]{$-1$};
    \draw[->,>=latex] (W3) -- (W9) node[above=0.25cm, midway]{$-1$};
    \draw[->,>=latex] (W4) -- (W10) node[above=0.25cm, midway]{$-1$};
    
    %Overbraces
\draw [decorate, decoration={brace,amplitude=5pt},xshift=6pt,yshift=0pt]
(6.4,-1.8) -- (6.4,-2.7) node [black,midway,right, xshift=5pt] {$4$ steps};

    \end{tikzpicture}
    \caption{The buildings blocks from \cref{infinitegadget} represented by
      black boxes are chained together ($n$ increases as you go to the
      right). The chain of white boxes allows to skip arbitrarily long
      prefixes while preserving path length. The positive rewards from the
      white states to the black boxes reimburse the lost reward accumulated
      until then. The $-1$ rewards between white states ensure that skipping
      gadgets forever is losing.
      \vspace{-5mm}
    }
    \label{chain}
\end{center}
    \end{figure}

%%%%%%%%%%%%%%%%%%%%%%%%%%%%%%%%%

The system consists of a sequence of gadgets.  \cref{infinitegadget} depicts a typical building block in this system. The system consists of these gadgets chained together as illustrated in  \cref{chain}, starting with $n$ sufficiently high at $n=N^{*}$. In the controlled choice, there is a small chance in all but the top choice of falling into a $\bot$ state. These $\bot$ states are abbreviations for an infinite chain of states with $-1$ reward on the transitions and are thus losing. The intuition behind the construction is that there is a random transition with branching degree $k(n)+1$. Then, the only way to win, in the controlled states, is to play the $i$-th choice if one arrived from the $i$-th choice. Thus intuitively, to remember what this choice was, one requires at least $k(n)+1$ memory modes. That is to say, the one and only way to win is to mimic, and mimicry requires memory.

\vspace{-2mm}
\begin{remark}
$\mathcal{M}$ is acyclic, finitely branching and for every state $s \in S, \exists n_{s} \in \N$ such that every path from $s_{0}$ to $s$ has length $n_{s}$. That is to say the step counter is implicit in the state. 
\end{remark}
Additionally, the number of transitions in each gadget now grows unboundedly with $n$ according to the function $k(n)$. Consequently, we will show that the number of memory modes required to play correctly grows above every finite bound. This will imply that no finite amount of memory suffices for $\varepsilon$-optimal strategies.

\noindent
\textbf{Notation}:
All logarithms are assumed to be in base $e$.
\begin{align*}
& \text{log}_{1}n \eqdef\text{log}n, \quad \text{log}_{i+1}n\eqdef\text{log}(\text{log}_{i}n) \\
& \delta_{0}(n)\eqdef\dfrac{1}{\text{log}n}, \quad \delta_{i}(n)\eqdef\dfrac{1}{\text{log}_{i+1}n}, 
\quad \delta_{k(n)}(n) \eqdef 1 - \sum_{j=0}^{k(n)-1}\delta_{j}(n)\\ 
& \varepsilon_{0}(n)\eqdef\dfrac{1}{n\text{log}n}, \ \varepsilon_{i+1}(n)\eqdef\dfrac{\varepsilon_{i}(n)}{\text{log}_{i+2}n}, 
\ \text{i.e. } \varepsilon_{i}(n) = \dfrac{1}{n \cdot \text{log}n \cdot \text{log}_{2}n \cdots \text{log}_{i+1}n},
\varepsilon_{k(n)}(n) \eqdef 0\\
& \text{Tower}(0) \eqdef e^{0} = 1, \quad \text{Tower}(i+1) \eqdef e^{\text{Tower}(i)},
\quad N_{i}\eqdef\text{Tower}(i)
\end{align*}

\begin{lemma}\label{convdiv}
The family of series 
$
\sum_{n > N_{j}} \delta_{j}(n)  \cdot \varepsilon_{i}(n) 
$
is divergent for all $i,j \in \N$, $i < j$. 

\noindent
Additionally, the related family of series 
$\sum_{n > N_{i}} \delta_{i}(n) \cdot \varepsilon_{i}(n)
$
is convergent for all $i \in \N.$

\end{lemma}
\begin{proof}
These are direct consequences of Cauchy's Condensation Test.
\end{proof}

\begin{definition}\label{def:kn}
We define $k(n)$, the rate at which the number of transitions grows. We define $k(n)$ in terms of fast growing functions $g, \text{Tower}$ and $h$ defined for $i \ge 1$ as follows:

$$
g(i) \eqdef \text{min} \left\{ N : \left( \sum_{n>N} \delta_{i-1}(n) \varepsilon_{i-1}(n) \right)  \le 2^{-i}  \right\}, \quad h(1) \eqdef 2
$$

%\eqdef \text{max} \left\{g(1), \text{Tower}(2), \text{min} \left\{ m \in \N :\sum^{m}_{n=2} \eps_{0}(n) \geq 1 \right\} \right\},

\vspace{-4mm}
$$
h(i+1) \eqdef \left\lceil \text{max} \left\{ g(i+1), \text{Tower}(i+2),  \text{min} \left\{ m+1 \in \N :\sum^{m}_{n=h(i)} \eps_{i-1}(n) \geq 1 \right\} \right\} \right\rceil.
$$

Note that function $g$ is well defined by \cref{convdiv}, and
$h(i+1)$ is well defined since for all $i$, $\sum^{\infty}_{n=h(i)} \eps_{i-1}(n)$ diverges to infinity.
$k(n)$ is a slow growing unbounded step function defined in terms of $h$ as 
$k(n) \eqdef h^{-1}(n)$.
%\text{min} \{ \lfloor g^{-1}(n) \rfloor, \lfloor \text{Tower}^{-1}(n+1) \rfloor \}.
The Tower function features in the definition to ensure that the transition probabilities are always well defined. $g$ and $h$ are used to smooth the proofs of \cref{infwin} and \cref{claim:divergence} respectively.
\emph{Notation:} $N^{*} \eqdef \text{min}\{n\in \N : k(n) = 1 \}$. This is intuitively the first natural number for which the construction is well defined.

The reward $m_{n}$ which appears in the $n$-th gadget is defined such that it 
outweighs any possible reward accumulated up to that
point in previous gadgets. As such we define $m_{n} \eqdef
2k(n) \sum_{i=N^{*}}^{n-1} m_{i},$ with $m_{N^{*}} \eqdef 1$ and where $k(n)$ is the
branching degree.
\end{definition}

To simplify the notation,
the state $s_{0}$ in our theorem statements refers to
$s_{N^{*}}$.

\begin{restatable}{lemma}{lemwelldefined}\label{lem:welldefined}
For $k(n) \geq 1$, the transition probabilities in the gadgets are well defined.
\end{restatable}
% \begin{proof}
% See \cref{sec:appreward}.
% \end{proof}

\newcommand{\lemwelldefinedproof}{
\begin{proof}
Recall that Tower$(i)$ is $i$ repeated exponentials. Thus, log(Tower$(i)$)=Tower$(i-1)$. 
%In a given gadget, the value of $n$ is at least Tower$(k(n)+1)$ by the definition of $k(n)$.

When checking whether probabilities in a given gadget are well defined,
first we choose a gadget. The choice of gadget gives us a branching degree
$k(n)+1$ which in turn lower bounds the value of $n$ in that gadget.
So for a branching degree of $k(n)+1$, we have $n$ lower bounded
by Tower$(k(n)+1)$ by definition of $k(n)$.

We need to show that $\sum_{i=0}^{k(n)-1} \delta_{i}(n) \leq 1$.
Indeed, we have that:

\[
\sum_{i=0}^{k(n)-1} \delta_{i}(n) 
\leq 
\sum_{i=0}^{k(n)-1} \dfrac{1}{\text{log}_{i+1}(\text{Tower}(k(n)+1))}\\
=
\sum_{i=1}^{k(n)} \dfrac{1}{\text{Tower}(i)}
<
\sum_{i=1}^{k(n)} \dfrac{1}{e^{i}}
<
\sum_{i=1}^{k(n)} \dfrac{1}{2^{i}}
<
1.
\]
Hence, for $k(n) \geq 1$, the transition probabilities are well defined, i.e.\ $\delta_{0}(n), \delta_{1}(n),...,\delta_{k(n)}(n)$ do indeed sum to 1.
\end{proof}
}

%%%%%%%%%%%%%%%%%%%%%%%%%%

%input{Figures/ijcase.tex}

%%%%%%%%%%%%%%%%%%%%%%%%%

\begin{restatable}{lemma}{lemmainfwin}\label{infwin}
For every $\varepsilon > 0$, there exists a strategy $\sigma_{\varepsilon}$ with  $\probm_{\mdp, s_{0}, \sigma_{\varepsilon}}(\liminfmpobj) \geq 1 - \eps$ that cannot fail unless it hits a $\perp$ state. Formally, $\probm_{\mdp, s_{0}, \sigma_{\varepsilon}}(\liminfmpobj \wedge \always(\neg \perp)) = \probm_{\mdp, s_{0}, \sigma_{\varepsilon}}(\always( \neg \perp)) \geq 1- \varepsilon$. So in particular, $\valueof{\mdp,\liminfmpobj}{s_{0}} = 1$.
\end{restatable}
\begin{proof}[Proof sketch]
(Full proof in \cref{sec:appreward}.)
We define a strategy $\sigma$ which in $c_{n}$ always mimics the choice in $s_{n}$. Playing according to $\sigma$, the only way to lose is by dropping into the $\perp$ state. This is because by mimicking, the player finishes each gadget with a reward of 0. From $s_{0}$, the probability of surviving while playing in all the gadgets is 
$$
\prod_{n \ge N^{*}} \left( 1 - \sum_{j=0}^{k(n)-1} \delta_{j}(n) \cdot \varepsilon_{j}(n) \right) > 0.
$$
Hence the player has a non zero chance of winning when playing $\sigma$.

When playing with the ability to skip gadgets, as illustrated in \cref{chain}, all runs not visiting a $\bot$ state are winning since the total reward never dips below $0$.
We then consider the strategy $\sigma_{\varepsilon}$ which plays like $\sigma$
after skipping forwards by sufficiently many gadgets
(starting at $n \gg N^{*}$). Its
probability of satisfying $\liminfmpobj$
corresponds to a tail of the
above product, which can be made arbitrarily close to $1$ (and thus $\ge 1-\eps$)
by \cref{prop:tail-product}.
Thus the strategies $\sigma_{\varepsilon}$ for arbitrarily small $\varepsilon >0$
witness that $\valueof{\mdp,\liminfmpobj}{s_{0}} = 1$.
\end{proof}

\newcommand{\lemmainfwinproof}{
\begin{proof}
We define a strategy $\sigma$ which in $c_{n}$ always mimics the choice in $s_{n}$. We first prove that playing this way gives us a positive chance of winning.
Then we show that there are strategies $\sigma_{\varepsilon}$ that attain $1- \varepsilon$ from $s_{0}$ without hitting a $\perp$ state.
This implies in particular that $\valueof{\mdp,\liminfmpobj}{s_{0}} = 1$.

 Playing according to $\sigma$, the only way to lose is by dropping into the $\perp$ state. This is because by mimicking, the player finishes each gadget with a reward of 0. In the $n$-th gadget, the chance of reaching the $\perp$ state is $\sum_{j=0}^{k(n)-1}\delta_{j}(n) \cdot \varepsilon_{j}(n)$. Thus, the probability of surviving while playing in all the gadgets is 
$$
\prod_{n \ge N^{*}} \left( 1 - \sum_{j=0}^{k(n)-1} \delta_{j}(n) \cdot \varepsilon_{j}(n) \right).
$$
However, by \cref{prop:product-sum}, this product is strictly greater than 0 if and only if the sum
$$
\sum_{n \ge N^{*}} \left( \sum_{i=0}^{k(n)-1} \delta_{i}(n) \varepsilon_{i}(n) \right)
$$
is finite. With some rearranging exploiting the definition of $k(n)$ we see that this is indeed the case:

\begin{align*}
& \sum_{n \ge N^{*}} \left( \sum_{i=0}^{k(n)-1} \delta_{i}(n) \varepsilon_{i}(n) \right) \\
\le & \sum_{i \ge 1} \left( \sum_{n=g(i)}^{\infty} \delta_{i-1}(n) \varepsilon_{i-1}(n) \right) &\text{by definition of $k(n)$}\\
\le & \sum_{i \ge 1} 2^{-i}  &\text{by definition of $g(n)$}\\
\le & 1
\end{align*}
Hence the player has a non zero chance of winning.

When playing with the ability to skip gadgets, as illustrated in \cref{chain},
all runs not visiting a $\perp$ state are winning since the total reward never
dips below $0$.
Hence $\probm_{\mdp, s_{0}, \sigma_{\varepsilon}}(\liminfmpobj \wedge \neg \perp) = \probm_{\mdp, s_{0}, \sigma_{\varepsilon}}( \neg \perp )$.
Thus the idea is to skip an arbitrarily long prefix of gadgets to push
the chance of winning $\varepsilon$ close to $1$ by pushing the
chance of visiting a $\perp$ state $\varepsilon$ close to $0$.
From the $N$-th state, for $N \ge N^{*}$, the chance of winning is
$$
\prod_{n \ge N} \left( 1 - \sum_{j=0}^{k(n)-1} \delta_{j}(n) \cdot \varepsilon_{j}(n) \right) > 0
$$
By \cref{prop:tail-product} this can be made arbitrarily close to $1$ by choosing $N$
sufficiently large.

Let 
$N_{\varepsilon} \eqdef \text{min} 
\left\{ 
    N \in \N \mid 
        \prod_{n \ge N} \left( 
            1 - \sum_{j=0}^{k(n)-1} \delta_{j}(n) \cdot \varepsilon_{j}(n) 
            \right) 
        \geq 1-\varepsilon 
\right\}$.
Now define the strategy $\sigma_{\varepsilon}$ to be the strategy that plays like $\sigma$ after skipping forwards by $N_{\varepsilon}$ gadgets. Thus, by definition $\sigma_{\varepsilon}$ attains $1-\varepsilon$ for all $\varepsilon > 0$.

Thus, by playing $\sigma_{\varepsilon}$ for an arbitrarily small $\varepsilon$ the chance of winning must be arbitrarily close to 1. Hence, $\valueof{\mdp,\liminfmpobj}{s_{0}} = 1$. 
\end{proof}
}

\smallskip
\begin{restatable}{lemma}{lemmainflose}\label{inflose}
For any FR strategy $\sigma$, almost surely either the mean payoff dips below $-1$ infinitely often, or the run hits a $\perp$ state, i.e.\ $\probm_{\mathcal{M}, \sigma, s_{0}}(\liminfmpobj)=0$.
\end{restatable}
\begin{proof}[Proof sketch]
(Full proof in \cref{sec:appreward}.)
Let $\sigma$ be some FR strategy with $k$ memory modes.
We prove a \emph{lower bound} $e_n$ on the probability of a local error
(reaching a $\bot$ state, or seeing a mean payoff $\le -1$)
in the current $n$-th gadget. This lower bound $e_n$ holds regardless
of events in past gadgets, regardless of the memory mode of $\sigma$
upon entering the $n$-th gadget, and cannot be improved by
$\sigma$ randomizing its memory updates.

The main idea is that,
once $k(n) > k+1$
(which holds for $n \ge N'$ sufficiently large)
by the Pigeonhole Principle there will always be
a memory mode confusing at least two different branches $i(n),j(n) \neq k(n)$
of the previous random choice at state $s_n$.
This confusion yields a probability $\ge e_n$
of reaching a $\bot$ state or seeing a mean payoff $\le -1$,
regardless of events in past gadgets and regardless
of the memory upon entering the $n$-th gadget.
We show that $\sum_{n \ge N'} e_n$ is a \emph{divergent} series.
Thus, by \cref{prop:product-sum}, $\prod_{n \ge N'} (1 - e_n)=0$. 
Hence, $\probm_{\mathcal{M}, \sigma, s_{0}}(\liminfmpobj) \le \prod_{n \ge
N'} (1 - e_n) = 0$.
\end{proof}

\cref{infwin} and \cref{inflose} yield the following theorem.

\begin{theorem} \label{infinitesummary}
There exists a countable, finitely branching and acyclic MDP $\mathcal{M}$ whose step counter is implicit in the state for which 
$\valueof{\mdp,\liminfmpobj}{s_{0}} = 1$ and any FR strategy $\sigma$ is such that 
$\probm_{\mdp, s_{0}, \sigma}(\liminfmpobj)=0$.
In particular, there are no $\varepsilon$-optimal $k$-bit Markov strategies
for any $k \in \N$ and any $\varepsilon < 1$ for
$\liminfmpobj$ in 
countable MDPs.
\end{theorem}
% \begin{proof}
% Proved by \cref{infwin} and \cref{inflose}.
% \end{proof}

All of the above results/proofs also hold for $\liminftpobj$, giving us the following theorem.

\begin{theorem} \label{infinitesummarytp}
There exists a countable, finitely branching and acyclic MDP $\mathcal{M}$ whose step counter is implicit in the state for which 
$\valueof{\mdp,\liminftpobj}{s_{0}} = 1$ and any FR strategy $\sigma$ is such that 
$\probm_{\mdp, s_{0}, \sigma}(\liminftpobj)=0$.
In particular, there are no $\varepsilon$-optimal $k$-bit Markov strategies
for any $k \in \N$ and any $\varepsilon < 1$ for
$\liminftpobj$ in 
countable MDPs.
\end{theorem}

\subsection{Optimal strategies}

%%%%%%%%%%%%%%%%%%%%%%%%%%%%%%%%%%%%%%%%%%%%%%%%%%%%%%

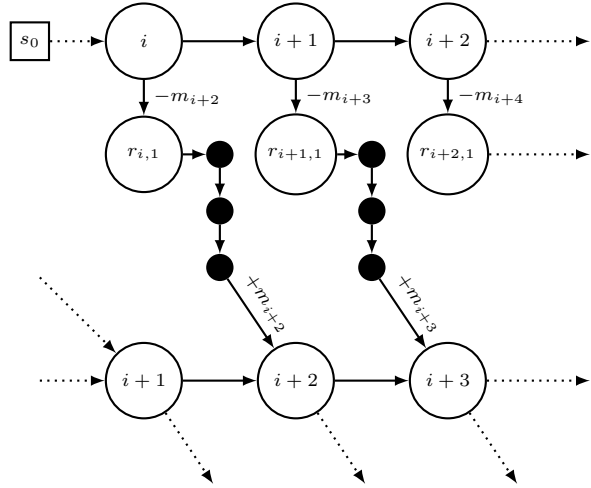
\begin{figure}[t]
\begin{center}
\begin{tikzpicture}

% invisible nodes
\node[draw, minimum height=0.5cm, minimum width=0.5cm] (I1) at (-0.5,0) {\scriptsize $s_{0}$};
\node[] (I2) at (-0.5,-4.5) {};
\node[] (I3) at (7,0) {};
\node[] (I4) at (7,-4.5) {};

% Invisible half way nodes
\node[] (HI1) at (-0.5,-3) {};
\node[] (HI2) at (7,-1.5) {};
\node[] (HI3) at (2,-6) {};
\node[] (HI4) at (4,-6) {};
\node[] (HI5) at (6,-6) {};

% the gadget nodes
\node[draw, circle, minimum width=1cm] (W1) at (1,0) {\scriptsize $i$};
\node[draw, circle, minimum width=1cm] (W2) at (3,0) {\scriptsize $i+1$};
\node[draw, circle, minimum width=1cm] (W3) at (5,0) {\scriptsize $i+2$};

\node[draw, circle, minimum width=1cm] (W4) at (1,-4.5) {\scriptsize $i+1$};
\node[draw, circle, minimum width=1cm] (W5) at (3,-4.5) {\scriptsize $i+2$};
\node[draw, circle, minimum width=1cm] (W6) at (5,-4.5) {\scriptsize $i+3$};

% Half way nodes
\node[draw, circle, minimum width=1cm] (R1) at (1,-1.5) {\scriptsize $r_{i, 1}$};
\node[draw, circle, minimum width=1cm] (R2) at (3,-1.5) {\scriptsize $r_{i+1, 1}$};
\node[draw, circle, minimum width=1cm] (R3) at (5,-1.5) {\scriptsize $r_{i+2, 1}$};

% Black nodes
\node[draw, circle, fill=black] (B1) at (2,-1.5) {};
\node[draw, circle, fill=black] (B2) at (2,-2.25) {};
\node[draw, circle, fill=black] (B3) at (2,-3) {};

\node[draw, circle, fill=black] (B4) at (4,-1.5) {};
\node[draw, circle, fill=black] (B5) at (4,-2.25) {};
\node[draw, circle, fill=black] (B6) at (4,-3) {};

% Drawing the dotted lines
\draw[->, >=latex, dotted, thick] (I1) -- (W1);
\draw[->, >=latex, dotted, thick] (I2) -- (W4);
\draw[->, >=latex, dotted, thick] (W3) -- (I3);
\draw[->, >=latex, dotted, thick] (HI1) -- (W4);
\draw[->, >=latex, dotted, thick] (R3) -- (HI2);
\draw[->, >=latex, dotted, thick] (W6) -- (I4);
\draw[->, >=latex, dotted, thick] (W4) -- (HI3);
\draw[->, >=latex, dotted, thick] (W5) -- (HI4);
\draw[->, >=latex, dotted, thick] (W6) -- (HI5);

% Drawing the edges
\draw[->,>=latex] 
(W1) edge node[right, midway]{\scriptsize $-m_{i+2}$} (R1)
(R1) edge (B1)
(B1) edge (B2)
(B2) edge (B3)
(B4) edge (B5)
(B5) edge (B6)
(B3) edge node[above, midway, sloped]{\scriptsize $+m_{i+2}$} (W5)
(W2) edge node[right, midway]{\scriptsize $-m_{i+3}$} (R2)
(R2) edge (B4)
(B6) edge node[above, midway, sloped]{\scriptsize $+m_{i+3}$}(W6)
(W3) edge node[right, midway]{\scriptsize $-m_{i+4}$} (R3);       
\draw[->,>=latex]
(W1) edge (W2)
(W2) edge (W3)
(W4) edge (W5)
(W5) edge (W6);

\end{tikzpicture}
\caption{Each row represents a copy of the MDP depicted in \cref{chain}.
  Each white circle labeled with a number $i$
  represents the correspondingly numbered gadget (like in \cref{infinitegadget})
  from that MDP. Now, instead
  of the bottom states in each gadget leading to an infinite losing chain,
  they lead to a restart state $r_{i,j}$ which leads to a fresh copy of the
  MDP (in the next row).
  Each restart incurs a penalty guaranteeing that the mean payoff dips
  below $-1$ before refunding it and continuing on in the next copy of the
  MDP. The states $r_{i,j}$ are labeled such that the $j$ indicates that if a
  run sees this state, then it is the $j$th restart. The $i$ indicates that
  the run entered the restart state from the $i$th gadget of the current copy
  of the MDP. The black states are dummy states inserted in order to preserve
  path length throughout.
\vspace{-5mm}
}
\label{restart}
\end{center}
\end{figure}

%%%%%%%%%%%%%%%%%%%%%%%%%%%%%%%%

Even for acyclic MDPs with the step counter implicit in the state,
optimal (and even almost sure winning) strategies for $\liminfmpobj$
require infinite memory.
To prove this,
we consider a variant of the MDP from the previous section
which has been augmented to include restarts from the $\perp$ states.
For the rest of the section, $\mathcal{M}$ is the MDP constructed in \cref{restart}.

\begin{remark}
$\mathcal{M}$ is acyclic, finitely branching and the step counter is implicit
in the state. We now refer to the rows of \cref{restart} as gadgets,
i.e., a gadget is a single instance of \cref{chain} where the $\perp$ states lead to the next row.
\end{remark}

\begin{restatable}{lemma}{lemmaalmostwin} \label{almostwin}
There exists a strategy $\sigma$ such that $\probm_{\mdp,\sigma,s_{0}}(\liminfmpobj)=1$.
\end{restatable}
\begin{proof}[Proof sketch]
(Full proof in \cref{sec:appreward}.)
Recall the strategy $\sigma_{1/2}$ defined in \cref{infwin} which achieves at least $1/2$ in each gadget that it is played in.
We then construct the almost surely winning strategy $\sigma$ by concatenating $\sigma_{1/2}$ strategies in the sense that $\sigma$ plays just like $\sigma_{1/2}$ in each gadget from each gadget's start state.

Since $\sigma$ achieves at least $1/2$ in every gadget that it sees, with probability 1, runs generated by $\sigma$ restart only finitely many times.
The intuition is then that a run restarting finitely many times must spend an infinite tail in some final gadget. Since $\sigma$ mimics in every controlled state, not restarting anymore directly implies that the total payoff is eventually always $\geq 0$. Hence all runs generated by $\sigma$ and restarting only finitely many times satisfy $\liminfmpobj$.
Therefore all but a nullset of runs generated by $\sigma$ are winning, i.e.\ $\probm_{\mdp,s_{0},\sigma}(\liminfmpobj)=1$. 
\end{proof}

\newcommand{\lemmaalmostwinproof}{
\begin{proof}

We will show that there exists a strategy $\sigma$ that satisfies the mean payoff objective with probability 1 from $s_{0}$. 
Towards this objective we recall the strategy $\sigma_{1/2}$ defined in \cref{infwin}. 
In a given gadget of this MDP with restarts, playing $\sigma_{1/2}$ in said gadget, there is a probability of at most 1/2 of restarting in that gadget. 
We then construct strategy $\sigma$ by concatenating $\sigma_{1/2}$ strategies in the sense that $\sigma$ plays just like $\sigma_{1/2}$ in each gadget from each gadget's start state.

Let $\playset$ be the set of runs induced by $\sigma$ from $s_{0}$. We partition $\playset$ into the sets $\playset_{i}$ and $\playset_{\infty}$ of runs such that 
$\playset = \left( \bigcup_{i=0}^{\infty} \playset_{i} \right) \cup \playset_{\infty}$. We define for $i=0$
$$\playset_{0} \eqdef \{ \rho \in \playset \mid 
    \forall \ell \in \N.\, \neg \eventually(r_{\ell, 1}) \},$$
for $i \geq 1$    
$$\playset_{i} \eqdef \{ \rho \in \playset \mid 
    \exists j \in \N.\, \eventually(r_{j,i}) \wedge \forall \ell \in \N.\, \neg \eventually(r_{\ell,i+1}) \}$$
and 
$$\playset_{\infty} \eqdef \{ \rho \in \playset \mid
    \forall i \in \N\ \exists j \in \N.\,  \eventually(r_{j,i})\}.$$
That is to say for all $i \in \N$, $\playset_{i}$ is the set of runs in $\playset$ that restart exactly $i$ times and $\playset_{\infty}$ is the set of runs in $\playset$ that restart infinitely many times.

We go on to define the sets of runs $\playset_{\geq i} \eqdef \bigcup_{j=i}^{\infty} \playset_{j} $ which are those runs which restart at least $i$ times. In particular note that $\playset_{\infty}= \bigcap_{i=0}^{\infty}\playset_{\geq i}$ and $\playset_{\geq i+1} \subseteq \playset_{\geq i}$.

By construction, any run $\rho \in \playset_{\infty}$ is losing since the negative reward that is collected upon restarting instantly brings the mean payoff below $-1$ by definition of $m_{n}$. Thus restarting infinitely many times translates directly into the mean payoff dropping below $-1$ infinitely many times and thus a strictly negative $\liminf$ mean payoff. As a result it must be the case that $\playset_{\infty} \subseteq \neg\liminfmpobj$.  

After every restart, the negative reward is reimbursed. Intuitively, going
through finitely many restarts does not damage the chances of winning.
We now show that, except for a nullset, the runs restarting only finitely many
times satisfy the objective.
Indeed, every run with only finitely many restarts must spend an infinite tail
in some final gadget in which it does not restart.
In this final gadget, the strategy plays just like $\sigma_{1/2}$, which means that it mimics the random choice in every controlled state. 
Since, by assumption, there are no more restarts, we obtain
$\probm_{\mdp,s_{0},\sigma}(\playset_{i}) = \probm_{\mdp,s_{0},\sigma}(\playset_{i} \wedge \forall j \in \N, \always (\neg r_{j,i+1}))$.
We then apply \cref{infwin} to obtain that
\begin{equation} \label{eq:alwayswin}
\probm_{\mdp,s_{0},\sigma}(\playset_{i}) = \probm_{\mdp,s_{0},\sigma}(\playset_{i} \wedge \forall j \in \N, \always (\neg r_{j,i+1})) = \probm_{\mdp,s_{0},\sigma}(\playset_{i} \wedge \liminfmpobj). 
\end{equation}
In other words, except for a nullset, the run restarting finitely often (here
$i$ times) satisfy $\liminfmpobj$.
Furthermore, notice that from this observation, the sets $\playset_{i}$
partition the set of winning runs.

We show now that $\probm_{\mdp,s_{0},\sigma}(\playset_{\infty}) = 0$. We do so firstly by showing by induction that $\probm_{\mdp,s_{0},\sigma}(\playset_{\geq i}) \leq 2^{-i}$ for $i \geq 1$, then applying the continuity of measures from above to obtain that $\probm_{\mdp,s_{0},\sigma}(\playset_{\infty}) = 0$.

\vspace{3mm}

Our base case is $i=1$. $\playset$, by definition of $\sigma$, is the set of
runs induced by playing $\sigma_{1/2}$ in every gadget. By \cref{infwin} $\sigma$
attains $\ge 1/2$ in every gadget. Therefore in particular the probability of a run leaving the first gadget is no more than $1/2$, i.e.\ $\probm_{\mdp,s_{0},\sigma}(\playset_{\geq 1}) \leq 1/2$.

Now suppose that $\probm_{\mdp,s_{0},\sigma}(\playset_{\geq i}) \leq 2^{-i}$. After restarting at least $i$ times, the probability of a run restarting at least once more is still $\leq 1/2$ since the strategy being played in every gadget is $\sigma_{1/2}$. Hence $$\probm_{\mdp,s_{0},\sigma}(\playset_{\geq i+1}) \leq \probm_{\mdp,s_{0},\sigma}(\playset_{\geq i}) \cdot \dfrac{1}{2} \leq 2^{-(i+1)}$$ which is what we wanted.

Now we use the fact that $\playset_{\infty}= \bigcap_{i=0}^{\infty}\playset_{\geq i}$ and $\playset_{\geq i+1} \subseteq \playset_{\geq i}$ to apply continuity of measures from above and obtain:
$$
\probm_{\mdp,s_{0},\sigma}(\playset_{\infty}) = \probm_{\mdp,s_{0},\sigma} \left( \bigcap_{i=0}^{\infty} \playset_{\geq i} \right) =
\lim_{i \to \infty} \probm_{\mdp,s_{0},\sigma}(\playset_{\geq i}) \leq \lim_{i \to \infty} 2^{-i} = 0.
$$
Hence $\playset_{\infty}$ is a null set.

We can now write down the following:
\begin{align*}
1 &= \probm_{\mdp,s_{0},\sigma}(\playset) \\
& = \left( \sum_{i=0}^{\infty} \probm_{\mdp,s_{0},\sigma}(\playset_{i}) \right) + \probm_{\mdp,s_{0},\sigma}(\playset_{\infty}) &\text{by partition of }\playset \\
& = \left( \sum_{i=0}^{\infty} \probm_{\mdp,s_{0},\sigma}(\playset_{i} \wedge \liminfmpobj) \right) + \probm_{\mdp,s_{0},\sigma}(\playset_{\infty}) & \text{by \cref{eq:alwayswin}} \\
& = \left( \sum_{i=0}^{\infty} \probm_{\mdp,s_{0},\sigma}(\playset_{i} \wedge \liminfmpobj) \right) + \probm_{\mdp,s_{0},\sigma}(\playset_{\infty} \wedge \liminfmpobj) & \text{by } \probm_{\mdp,s_{0},\sigma}(\playset_{\infty}) = 0 \\
& = \probm_{\mdp,s_{0},\sigma}(\liminfmpobj) &\text{by partition of }\liminfmpobj
\end{align*}

Thus $\probm_{\mdp,s_{0},\sigma}(\playset) = \probm_{\mdp,s_{0},\sigma}(\liminfmpobj) = 1$, i.e.\ $\sigma$ wins almost surely.
\end{proof}
}

\begin{restatable}{lemma}{lemmaalmostlose} \label{almostlose}
For any FR strategy $\sigma$, $\probm_{\mathcal{M}, \sigma, s_{0}}(\liminfmpobj)=0$.
\end{restatable}

\begin{proof}[Proof sketch]
(Full proof in \cref{sec:appreward}.)
% There are two ways to lose when playing in this MDP: either the mean payoff dips below $-1$ infinitely often because the run takes infinitely many restarts, or the run only takes finitely many restarts, but the mean payoff drops below $-1$ infinitely many times in the last copy of the gadget that the run stays in. Recall that in \cref{inflose} we showed that any FR strategy with probability 1 either restarts or lets the mean payoff dip below $-1$ infinitely often. 
%
Let $\sigma$ be any FR strategy. We partition the runs generated by $\sigma$ into runs restarting infinitely often, and those restarting only finitely many times. Any runs restarting infinitely often are losing by construction. Those runs restarting only finitely many times, once in the gadget they spend an infinite tail in, let the mean payoff dip below $-1$ infinitely many times with probability 1 by \cref{inflose}.
Hence we have that $\probm_{\mathcal{M}, \sigma, s_{0}}(\liminfmpobj)=0$.
\end{proof}

\newcommand{\lemmaalmostloseproof}{
\begin{proof}
There are two ways to lose when playing in this MDP: either the mean payoff dips below $-1$ infinitely often because the run takes infinitely many restarts, or the run only takes finitely many restarts, but the mean payoff drops below $-1$ infinitely many times in the last copy of the gadget that the run stays in. Recall that in \cref{inflose} we showed that any FR strategy with probability 1 either restarts or lets the mean payoff dip below $-1$ infinitely often. 

Let $\sigma$ be any FR strategy and let $\playset$ to be the set of runs induced by $\sigma$ from $s_{0}$.
We partition $\playset$ into the sets $\playset_{i}$ and $\playset_{\infty}$ of runs such that 
$\playset = \left( \bigcup_{i=0}^{\infty} \playset_{i} \right) \cup \playset_{\infty}$. Where we define for $i=0$
$$\playset_{0} \eqdef \{ \rho \in \playset \mid 
    \forall \ell \in \N, \neg \eventually(r_{\ell,1}) \},$$
for $i \geq 1$    
$$\playset_{i} \eqdef \{ \rho \in \playset \mid 
    \exists j \in \N, \eventually(r_{j, i}) \wedge \forall \ell \in \N, \neg \eventually(r_{\ell,i+1}) \}$$
and 
$$\playset_{\infty} \eqdef \{ \rho \in \playset \mid
    \forall i, \exists j  \text{ F}(r_{j,i})\}.$$
That is to say for all $i \in \N$, $\playset_{i}$ is the set of runs in $\playset$ that restart exactly $i$ times and $\playset_{\infty}$ is the set of runs in $\playset$ that restart infinitely many times.

We go on to define the sets of runs $\playset_{\geq i} \eqdef \bigcup_{j=i}^{\infty} \playset_{j} $ which are those runs which restart at least $i$ times. In particular note that $\playset_{\infty}= \bigcap_{i=0}^{\infty}\playset_{\geq i}$ and $\playset_{\geq i+1} \subseteq \playset_{\geq i}$.

Note that any run in $\playset_{\infty}$ is losing by construction. The negative reward that is collected upon restarting instantly brings the mean payoff below $-1$ by definition of $m_{n}$. Thus restarting infinitely many times translates directly into the mean payoff dropping below $-1$ infinitely many times. Thus $\playset_{\infty} \subseteq \neg\liminfmpobj$ and so it follows that  $\probm_{\mdp,s_{0},\sigma}(\playset_{\infty}) = \probm_{\mdp,s_{0},\sigma}(\playset_{\infty} \wedge \neg \liminfmpobj)$. Since the sets $\playset_{i}$ and $\playset_{\infty}$ partition $\playset$ we have that:
$$
\probm_{\mdp,s_{0},\sigma}(\playset) = \left( \sum_{i=0}^{\infty} \probm_{\mdp,s_{0},\sigma}(\playset_{i}) \right) + \probm_{\mdp,s_{0},\sigma}(\playset_{\infty}).
$$

It remains to show that every set $\playset_{i}$ is almost surely losing, i.e.\ $\probm_{\mdp,s_{0},\sigma}(\playset_{i}) = \probm_{\mdp,s_{0},\sigma}(\playset_{i} \wedge \neg \liminfmpobj).$
Consider a run $\rho \in \playset_{i}$. By definition it restarts exactly $i$ times. As a result, it spends infinitely long in the $i+1$st gadget. 
Because $\sigma$ is an FR strategy, it must be the case that any substrategy $\sigma^{*}$ induced by $\sigma$ that is played in a given gadget is also an FR strategy. 
This allows us to apply \cref{inflose} to obtain that
\begin{equation}\label{eq:alwayslose}
\probm_{\mdp,s_{0},\sigma}(\playset_{i}) = 
\probm_{\mdp,s_{0},\sigma}\left(\playset_{i} \wedge (\neg \liminfmpobj \vee \exists j \in \N, 
\eventually (r_{j,i+1}))\right).
\end{equation}
However, any run $\rho \in \playset_{i}$ never sees any state $r_{j,i+1}$ for any $j$ by definition. Therefore it follows that 
$$
\probm_{\mdp,s_{0},\sigma}\left(\playset_{i} \wedge (\neg \liminfmpobj \vee \exists j \in \N, \eventually (r_{j,i+1}))\right)  = 
\probm_{\mdp,s_{0},\sigma}\left(\playset_{i} \wedge (\neg \liminfmpobj )\right) 
$$ 
Hence $\probm_{\mdp,s_{0},\sigma}(\playset_{i}) = \probm_{\mdp,s_{0},\sigma}(\playset_{i} \wedge \neg \liminfmpobj)$ as required.

As a result we have that 
\begin{align*}
1 &= \probm_{\mdp,s_{0},\sigma}(\playset) \\
& = \left( \sum_{i=0}^{\infty} \probm_{\mdp,s_{0},\sigma}(\playset_{i}) \right) + \probm_{\mdp,s_{0},\sigma}(\playset_{\infty})    &\text{by partition of }\playset \\
& = \left( \sum_{i=0}^{\infty} \probm_{\mdp,s_{0},\sigma}(\playset_{i} \wedge \neg\liminfmpobj) \right) + \probm_{\mdp,s_{0},\sigma}(\playset_{\infty} \wedge \neg\liminfmpobj) &\text{by \cref{eq:alwayslose}}\\
& = \probm_{\mdp,s_{0},\sigma}(\neg \liminfmpobj) &\text{by partition of }\playset
\end{align*}

That is to say that for any FR strategy $\sigma$, $\probm_{\mdp,s_{0},\sigma}(\liminfmpobj)=0$.
\end{proof} 
}

From \cref{almostwin} and \cref{almostlose} we obtain the following theorem.

\begin{theorem} \label{almostsummary}
There exists a countable, finitely branching and acyclic MDP $\mathcal{M}$ whose step counter is implicit in the state for which 
$\state_0$ is almost surely winning $\liminfmpobj$, i.e.,
$\exists\hat{\sigma}\,\probm_{\mdp, s_{0}, \hat{\sigma}}(\liminfmpobj)=1$,
but every FR strategy $\sigma$ is such that 
$\probm_{\mdp, s_{0}, \sigma}(\liminfmpobj)=0$.
In particular, almost sure winning strategies, when they exist, cannot be chosen 
$k$-bit Markov for any $k \in \N$ for countable MDPs.
\end{theorem}
% \begin{proof}
% Proved by \cref{almostwin} and \cref{almostlose}.
% \end{proof}

All of the above results/proofs also hold for $\liminftpobj$, giving us the following theorem.

\begin{theorem} \label{almostsummarytp}
There exists a countable, finitely branching and acyclic MDP $\mathcal{M}$ whose step counter is implicit in the state for which 
$\state_0$ is almost surely winning $\liminftpobj$, i.e.,
$\exists\hat{\sigma}\,\probm_{\mdp, s_{0}, \hat{\sigma}}(\liminftpobj)=1$,
but every FR strategy $\sigma$ is such that 
$\probm_{\mdp, s_{0}, \sigma}(\liminftpobj)=0$.
In particular, almost sure winning strategies, when they exist, cannot be chosen 
$k$-bit Markov for any $k \in \N$ for countable MDPs.
\end{theorem}

\section{When is a reward counter not sufficient?}\label{sec:liminfstep}
In this part we show that a reward counter plus arbitrary finite memory
does not suffice for ($\eps$-)optimal strategies for $\liminfmpobj$,
even if the MDP is finitely branching.

The same lower bound holds for $\liminftpobj$/$\liminfppobj$,
but only in infinitely branching MDPs. The finitely branching case is
different for $\liminftpobj$/$\liminfppobj$; cf.~\cref{sec:upper}.

The techniques used to prove
these results are similar to those in \cref{sec:liminfreward}
and proofs can be found in \cref{app:step}.

\begin{restatable}{theorem}{thmmpstepepslower}\label{mpstepepslower}
There exists a countable, finitely branching, acyclic MDP $\mdp_{\text{\emph{RI}}}$ with initial state $(s_{0},0)$ with the total reward implicit in the state such that 
\begin{itemize}
\item $\valueof{\mdp_{\text{\emph{RI}}}, \liminfmpobj}{(s_{0},0)} = 1$,
\item for all FR strategies $\sigma$, we have $\probm_{\mdp_{\text{\emph{RI}}}, (s_{0},0), \sigma}(\liminfmpobj)=0$.
\end{itemize}
% Hence, $\eps$-optimal strategies for $\liminfmpobj$ objectives require at least a step counter.
\end{restatable}

\begin{restatable}{theorem}{thmmpstepoptlower}\label{mpstepoptlower}
There exists a countable, finitely branching and acyclic MDP $\mdp_{\text{\emph{Restart}}}$
whose total reward is implicit in the state where, for the initial state $s_0$,
\begin{itemize}
\item
there exists an HD strategy $\sigma$ s.t.\
$\probm_{\mdp_{\text{\emph{Restart}}}, s_{0}, \sigma}(\liminfmpobj)=1$.
\item
for every FR strategy $\sigma$,
$\probm_{\mdp_{\text{\emph{Restart}}}, s_{0}, \sigma}(\liminfmpobj)=0$.
\end{itemize}
\end{restatable}

\begin{restatable}{theorem}{thminfbranchsteplower}\label{infbranchsteplower}
There exists an infinitely branching MDP $\mdp$ as in \cref{infinitebranchtp} with reward implicit in the state and initial state $s$ such that
\begin{itemize}
\item every FR strategy $\sigma$ is such that $\probm_{\mdp, s, \sigma} (\liminftpobj) = 0$ and $\probm_{\mdp, s, \sigma} (\liminfppobj) = 0$
\item there exists an HD strategy $\sigma$ s.t.\ $\probm_{\mdp, s, \sigma} (\liminftpobj) = 1$ and $\probm_{\mdp, s, \sigma} (\liminfppobj) = 1$.
\end{itemize}
Hence, optimal (and even almost-surely winning) strategies and $\eps$-optimal
strategies for $\liminftpobj$ and $\liminfppobj$ require infinite memory
beyond a reward counter.
\end{restatable}

\begin{remark}\label{rem:glue}
The MDPs from \cref{sec:liminfreward}
and \cref{sec:liminfstep}
show that good strategies for $\liminfmpobj$ require at least
(in the sense of \cref{rem:lowerbonds})
a reward counter and a step counter, respectively.
There does, of course, exist a \emph{single MDP}
where good strategies for $\liminfmpobj$
require at least both a step
counter and a reward counter. We construct such an MDP by `gluing' the two
different MDPs together via an initial random state which points to each with
probability $1/2$.
% As such, $\eps$-optimal strategies would require both a step counter and a reward counter in order to do well, since a single strategy has to cope with the constraints of both MDPs.
\end{remark}

\section{Upper bounds}\label{sec:upper}
We establish upper bounds on the strategy complexity
of $\liminf$ threshold objectives for mean payoff, total payoff and point payoff. 
It is noteworthy that once the reward structure of an MDP has been encoded into the states,
then these threshold objectives take on a qualitative flavor not
dissimilar to Safety or co-B\"{u}chi (cf.~\cite{KMSW2017}).
Indeed, if the transition rewards are restricted to integer values,
then $\liminftpobj$ boils down to eventually avoiding all transitions with
negative reward (since negative rewards would be $\le -1$).
This is a co-B\"{u}chi objective.
However, if the rewards are not restricted to integers, then the picture is not so simple.

For \emph{finitely branching} MDPs, we show that there exist $\eps$-optimal MD
strategies for $\liminfppobj$.
In turn, this yields the requisite upper bound for finitely branching
$\liminftpobj$, i.e., using just a reward counter.

For \emph{infinitely branching} MDPs, a step counter suffices in
order to achieve $\liminfppobj$ $\eps$-optimally.
Then, by encoding the total reward into the states, this will also give us
SC+RC upper bounds for $\liminfmpobj$ as well as infinitely branching $\liminftpobj$
(i.e., using both a step counter and a reward counter).

First we show how to encode the total reward level into the state in
a given MDP.

\begin{remark}
Given an MDP $\mdp$ and initial state $s_{0}$, we can construct
an MDP $R(\mdp)$ with initial state $(s_{0},0)$ and with the reward counter
implicit in the state such that strategies in $R(\mdp)$ can be translated back
to $\mdp$ with an extra reward counter; cf.~\cref{def:encodereward} for a
formal definition.
\end{remark}

\newcommand{\defencodereward}{
\begin{definition}\label{def:encodereward}
Let $\mdp$ be an MDP. From a given initial state $s_{0}$,
the reward level in each state $s \in S$ can be any of the
countably many values $r_{1}, r_{2}, \dots$
corresponding to the rewards accumulated along all the possible paths
leading to $s$ from $s_{0}$.
We then construct the MDP $R(\mdp) \eqdef (S', \zstates', \rstates', \longrightarrow_{R(\mdp)}, P')$ as follows:
\begin{itemize}
\item
The state space of $R(\mdp)$ is 
$S' \eqdef \{ (s,r) \mid s \in S \text{ and } r \in \mathbb{R} \text{ is a reward level attainable at } s \}$.
Note that $S'$ is countable.
We write $s_{0}'$ for the initial state $(s_{0},0)$.
\item
$\zstates' \eqdef \{ (s,r) \in S' \mid s \in \zstates \}$
and $\rstates' \eqdef S' \setminus \zstates'$.
\item
The set of transitions in $R(\mdp)$ is 
\begin{align*}
\longrightarrow_{R(\mdp)} \eqdef 
\{ &
\left( (s,r),(s',r') \right) \mid (s,r),(s',r') \in S', \\
& s \longrightarrow s' \text{ in } \mdp \text{ and }  r' \eqdef
r+r(s \to s')
% \text{ where $t$ is the reward for taking transition } s \longrightarrow s'
\}.
\end{align*}
% If $\left( (s,r),(s',r') \right) \in \longrightarrow_{R(\mdp)} $ then we simply write $(s,r) \longrightarrow (s',r')$.

\item $P': \rstates' \to \mathcal{D}(S')$ is defined such that 
\[
P'(s,r)(s',r') \eqdef  
    \begin{cases}
    P(s)(s') & \text{ if } (s,r) \longrightarrow_{R(\mdp)} (s',r') \\
    0 & \text{ otherwise }
    \end{cases}
\]

\item The reward for taking transition $ (s,r) \longrightarrow (s',r')$ is $r'$.
\end{itemize}
\end{definition}
}

By labeling transitions in $R(\mdp)$ with the state encoded total reward of the
target state, we ensure that the
point rewards in $R(\mdp)$ correspond exactly to the total rewards in $\mdp$.

% This lemma does not hold for general strategies, but still for MD/Markov strategies.
\begin{restatable}{lemma}{lemmatotaltopoint}\label{totaltopoint}
Let $\mdp$ be an MDP with initial state $s_{0}$. Then given an MD
(resp.\ Markov) strategy $\sigma'$ in $R(\mdp)$ attaining $c \in [0,1]$ for
$\liminfppobj$ from $(s_{0},0)$, there exists a strategy
$\sigma$ attaining $c$ for $\liminftpobj$ in $\mdp$ from $s_{0}$ which uses the same memory as $\sigma'$ plus a reward counter.
\end{restatable}
% \begin{proof}
% See \cref{app:upper}.
% \end{proof}

\newcommand{\lemmatotaltopointproof}{
\begin{proof}
Let $\sigma'$ be an MD (resp.\ Markov) strategy in $R(\mdp)$ attaining $c \in [0,1]$ for
$\liminfppobj$ from $(s_{0},0)$.
We define a strategy $\sigma$ on $\mdp$ from $s_0$ that uses the same memory
as $\sigma'$ plus a reward counter. Then $\sigma$ plays on $\mdp$ exactly like
$\sigma'$ plays on $R(\mdp)$, keeping the reward counter in its memory instead
of in the state.
I.e., at a given state $s$ (and step counter value $m$, in case $\sigma'$ was a
Markov strategy) and reward level $r$, $\sigma$ plays exactly as $\sigma'$
plays in state $(s,r)$ (and step counter value $m$, in case $\sigma'$ was a
Markov strategy).
% $\sigma(s,r,m) = (s',r',m') = \sigma' ((s,r),m)((s',r'),m')$.
By our construction of $R(\mdp)$ and the definition of $\sigma$,
the sequences of point rewards seen by $\sigma'$ in runs on $R(\mdp)$
coincide with the sequences of total rewards seen by $\sigma$ in runs in $\mdp$.
Hence we obtain
$\probm_{R(\mdp),(s_0,0),\sigma'}(\liminfppobj)
= \probm_{\mdp,s_0,\sigma}(\liminftpobj)$
as required.
\end{proof}
}
% \lemmatotaltopointproof

\begin{remark}\label{steptomarkov}
Given an MDP $\mdp$ and initial state $s_{0}$, we can construct
an acyclic MDP $S(\mdp)$ with initial state
$(s_{0},0)$ and with the step counter implicit in the state such that
MD strategies in $S(\mdp)$ can be translated back to $\mdp$ with the
use of a step counter to yield deterministic Markov strategies in $\mdp$;
cf.~\cite[Lemma 4]{KMST2020c}.
\end{remark}

\begin{remark}\label{remdefam}
In order to tackle the mean payoff objective $\liminfmpobj$ on $\mdp$, we
define a new acyclic MDP $A(\mdp)$ which encodes both the step counter and the
average reward into the state. 
However, since we want the point rewards in $A(\mdp)$ to coincide
with the \emph{mean payoff} in the original MDP $\mdp$, the transition rewards in $A(\mdp)$ are
given as the encoded rewards divided by the step counter (unlike in
$R(\mdp)$); cf.~\cref{def:encodeam} for a formal definition.
\end{remark}

\newcommand{\defencodeam}{
\begin{definition}\label{def:encodeam}
Given an MDP $\mdp$ with initial state $s_{0}$, we define the new MDP $A(\mdp)$.
From the initial state $s_{0}$, the reward level in each state $s \in S$
can be any of the countably many values $r_{1}, r_{2}, \dots$
corresponding to the rewards accumulated along all the possible paths leading to $s$ from $s_{0}$.

We then construct $A(\mdp) \eqdef (S', \zstates', \rstates', \longrightarrow_{A(\mdp)}, P')$ as follows:
\begin{itemize}
\item
The state space of $A(\mdp)$ is 
\[
S' \eqdef \{(s,n,r) \mid s \in S, n \in \mathbb{N}
\text{ and } r \in \!\mathbb{R}\, \text{ is a reward level attainable at $s$ at step
$n$}\}
\]
Note that $S'$ is countable.
We write $s_{0}'$ for the initial state $(s_{0},0,0)$ of $A(\mdp)$.
\item
$ \zstates' \eqdef 
\{(s,n,r) \in S' \mid s \in \zstates \}$ and
$\rstates' \eqdef S' \setminus \zstates'$.
\item
The set of transitions in $A(\mdp)$ is 
\begin{align*}
\longrightarrow_{A(\mdp)} \eqdef 
\{ &
\left( (s,n,r),(s',n+1,r') \right) \mid \\
& (s,n,r),(s',n+1,r') \in S',\\
& s \longrightarrow s' \text{ in } \mdp \text{ and } 
r'= r+r(s \rightarrow s')
\}.
\end{align*}
% If $\left( (s,n,r),(s',n+1,r') \right) \in \longrightarrow_{A(\mdp)} $ then we simply write $(s,n,r) \longrightarrow (s',n+1,r')$.
\item $P': \rstates' \to \mathcal{D}(S')$ is defined such that 
\[
P'(s,n,r)(s',n',r') \eqdef  
    \begin{cases}
    P(s)(s') & \text{if } (s,n,r) \!\rightarrow_{A(\mdp)}\! (s',n',r') \\
    0 & \!\!\text{otherwise}
    \end{cases}
\]

\item The reward for taking transition $ (s,n,r) \longrightarrow (s',n',r')$ is $r' / n'$.

\end{itemize}
\end{definition}
}
% \defencodeam

% By labeling the transitions in $A(\mdp)$ with the average reward of the
% target state,
% we ensure that the sequence of point rewards in $A(\mdp)$ is the same as the
% sequence of seen mean payoffs in $\mdp$.

\begin{lemma}\label{meantopoint}
Let $\mdp$ be an MDP with initial state $s_{0}$.
Then given an MD strategy $\sigma'$ in $A(\mdp)$ attaining $c \in [0,1]$ for
$\liminfppobj$ from $(s_{0},0,0)$, there exists a strategy $\sigma$ attaining
$c$ for $\liminfmpobj$ in $\mdp$ from $s_{0}$ which uses just a reward counter
and a step counter.
\end{lemma}
\begin{proof}
The proof is very similar to that of \cref{totaltopoint}.
\end{proof}

\begin{lemma}(\cite[Lemma 23]{KMST2020c})
\label{acyclicsafety}
For every acyclic MDP with a safety objective and every $\eps > 0$,
there exists an MD strategy that is uniformly $\eps$-optimal.
\end{lemma}

\begin{theorem}(\cite[Theorem 7]{KMST:Transient-arxiv})
\label{epsilontooptimal}
Let $\mdp=\mdptuple$ be a countable MDP, and let $\formula$ be an event that is tail in~$\mdp$.
Suppose for every $s \in S$ there exist $\eps$-optimal MD strategies for~$\formula$.
Then:
\begin{description}
\item[1.] There exist uniform $\eps$-optimal MD strategies for~$\formula$.
\item[2.] There exists a single MD strategy that is optimal from every state that has an optimal strategy.
\end{description}
\end{theorem}

\subsection{Finitely Branching Case}\label{subsec:upper-fb}

In order to prove the main result of this section, we use the following result
on the $\transience$ objective, which is the set of runs that do not visit any state infinitely often.
Given an MDP $\mdp=\mdptuple$, 
$\transience \eqdef \bigwedge_{s \in S} \eventually \always \neg s$.

\begin{theorem}(\cite[Theorem 8]{KMST:Transient-arxiv})
\label{epstransience}
In every countable MDP there exist uniform $\eps$-optimal MD strategies for
$\transience$.
\end{theorem}

%We need the following technical lemma that holds only for finitely
% branching MDPs.
% (proof in \cref{app:upper}).

\newcommand{\lemfbavoidstatement}{
\begin{restatable}{lemma}{lemfbavoid}\label{lem:fbavoid}
Given a finitely branching countable MDP $\mdp$, a subset $T \subseteq \to$ of
the transitions and a state $\state$, we have
\[
\valueof{\mdp,\neg\eventually T}{\state} < 1
\ \Rightarrow\ \exists k \in \N.\,
\valueof{\mdp,\neg\eventually^{\le k} T}{\state} < 1 
\]
i.e., if it is impossible to completely avoid $T$ then
there is a bounded threshold $k$ and a fixed nonzero
chance of seeing $T$ within $\le k$ steps, regardless of the strategy.
\end{restatable}
}
\newcommand{\lemfbavoidproof}{
\begin{proof}
If suffices to show that
$\forall k \in \N.\, \valueof{\mdp,\neg\eventually^{\le k} T}{\state} =1$
implies $\valueof{\mdp,\neg\eventually T}{\state} = 1$.
Since $\mdp$ is finitely branching, the state $\state$ has only finitely many
successors $\{\state_1,\dots,\state_n\}$.

Consider the case where $\state$ is a controlled state.
If we had the property $\forall {1 \le i\le n}\,\exists k_i \in \N.\,
\valueof{\mdp,\neg\eventually^{\le k_i} T}{\state_{i}} < 1$
then we would have
$\valueof{\mdp,\neg\eventually^{\le k} T}{\state} < 1$
for $k=(\max_{1 \le i \le n} k_i)+1$
which contradicts our assumption.
Thus there must exist an $i \in \{1,\dots,n\}$ with
$\forall k \in \N.\, \valueof{\mdp,\neg\eventually^{\le k} T}{\state_i} =1$.
We define a strategy $\zstrat$ that chooses the successor state $s_i$ when in
state $\state$.

Similarly, if $\state$ is a random state, we must have
$\forall k \in \N.\, \valueof{\mdp,\neg\eventually^{\le k} T}{\state_i} =1$
for all its successors $s_i$.

By using our constructed strategy $\zstrat$, we obtain
$\probm_{\mdp,\state,\zstrat}(\neg\eventually T)=1$ and thus
$\valueof{\mdp,\neg\eventually T}{\state} = 1$ as required.
\end{proof}
}
% \lemfbavoidproof

\begin{theorem}\label{finpointpayoff}
Consider a finitely branching MDP $\mdp =\mdptuple$ with initial state $s_{0}$ and a $\liminfppobj$ objective. Then there exist $\eps$-optimal MD strategies.
\end{theorem}
\begin{proof}
Let $\eps >0$.
We begin by partitioning the state space into two sets, $S_{\text{safe}}$ and
$S \setminus S_{\text{safe}}$.
The set $S_{\text{safe}}$ is the subset of states which is surely winning for
the safety objective of only using transitions with non-negative rewards
(i.e., never using transitions with negative rewards at all).
Since $\mdp$ is finitely branching, there exists a uniformly optimal MD
strategy $\sigma_{\text{safe}}$ for this safety objective
\cite{Puterman:book,KMSW2017}.

We construct a new MDP $\mdp'$ by modifying $\mdp$. We create a gadget
$G_{\text{safe}}$ composed of a sequence of new controlled states
$x_{0}, x_{1}, x_{2}, \dots$ where all  transitions $x_{i} \to x_{i+1}$
have reward $0$. Hence any run entering $G_{\text{safe}}$ is winning for $\liminfppobj$. 
We insert $G_{\text{safe}}$ into $\mdp$ by replacing all incoming transitions
to $S_{\text{safe}}$ with transitions that lead to $x_{0}$. The idea behind
this construction is that when playing in $\mdp$, once you hit a state in
$S_{\text{safe}}$, you can win surely by playing an optimal MD strategy for
safety. So we replace $S_{\text{safe}}$ with the surely winning gadget
$G_{\text{safe}}$.
Thus
\begin{equation}\label{eq:valuesmmprime}
\valueof{\mdp,\liminfppobj}{s_0} = \valueof{\mdp',\liminfppobj}{s_0} 
\end{equation}
and if an $\eps$-optimal MD strategy exists in $\mdp$,
then there exists a corresponding one in $\mdp'$, and vice-versa.

We now consider a general (not necessarily MD) $\eps$-optimal strategy $\sigma$
for $\liminfppobj$ from $s_{0}$ on $\mdp'$, i.e.,
\begin{equation}\label{eq:fbsigma-eps-opt}
\probm_{\mdp',s_0,\sigma}(\liminfppobj) \ge \valueof{\mdp',\liminfppobj}{s_0} - \eps.
\end{equation}
Define the safety objective $\text{Safety}_{i}$ which is the objective of
never seeing any point rewards $< -2^{-i}$.
This then allows us to characterize $\liminfppobj$ in terms of safety objectives.
\begin{equation} \label{eq:fbliminfppissafety}
\liminfppobj = \bigcap_{i \in \mathbb{N}} \eventually(\text{Safety}_{i}).
\end{equation}

Now we define the safety objective $\text{Safety}_{i}^{k} \eqdef \eventually^{\leq k}( \text{Safety}_{i} )$ to attain $\text{Safety}_{i}$ within at most $k$ steps. This allows us to write 
\begin{equation} \label{eq:fbsafetyisunion}
\eventually(\text{Safety}_{i}) = \bigcup_{k \in \mathbb{N}} \text{Safety}_{i}^{k}.
\end{equation}
By continuity of measures from above we get 
%\begin{align*}
\[
0 = \mathcal{P}_{\mdp',s_0,\sigma} \left( \eventually(\text{Safety}_{i}) \cap \bigcap_{k \in \N} \overline{\text{Safety}^{k}_{i}} \right)
  = \lim_{k \to \infty} \mathcal{P}_{\mdp',s_0,\sigma} \left( \eventually(\text{Safety}_{i}) \cap \overline{\text{Safety}^{k}_{i}}
\right).
\]
%\end{align*}
Hence for every $i \in \mathbb{N}$ and
$\eps_{i} \defeq \eps \cdot 2^{-i}$
there exists $n_{i}$ such that
\begin{equation}\label{eq:fbni}
\mathcal{P}_{\mdp',s_0,\sigma} \left( \eventually(\text{Safety}_{i}) \cap
\overline{\text{Safety}^{n_{i}}_{i}} \right) \leq \eps_{i}.
\end{equation}

Now we can show the following claim (proof in \cref{app:upper-fb}).
\begin{restatable}{claim}{claimfblosetwoeps}\label{claim:fblose2eps}
\[
\probm_{\mdp',s_0,\sigma} \left( \bigcap_{i \in \N} \text{Safety}^{n_{i}}_{i}  \right) 
\geq
\valueof{\mdp',\liminfppobj}{s_0} - 2\eps.
\]
\end{restatable}

Since $\mdp'$ does not have an implicit step counter, we use the following
construction to approximate one.
We define the distance $d(s)$ from $s_{0}$ to a state $s$ as the length of the shortest path from $s_{0}$ to $s$. 
Let $\text{Bubble}_{n}(s_{0}) \defeq \{s \in S \mid d(s) \leq n\}$
be those states that can be reached within $n$ steps from $s_{0}$.
Since $\mdp'$ is finitely branching, $\text{Bubble}_{n}(s_{0})$ is finite for
every fixed $n$.
Let 
%\begin{align*}
\[
\text{Bad}_{i} \defeq \{  t \in \longrightarrow_{\mdp'} \mid t =
s \longrightarrow_{\mdp'} s',
s \notin \text{Bubble}_{n_{i}}(s_{0}) \text{ and } r(t) < -2^{-i}
\}
%\end{align*}
\]
be the set of transitions originating outside $\text{Bubble}_{n_{i}}(s_{0})$ whose reward is too negative. 
Thus a run from $s_0$ that satisfies $\text{Safety}_{i}^{n_i}$ cannot
use any transition in $\text{Bad}_{i}$, since (by definition of
$\text{Bubble}_{n_{i}}(s_{0})$) they would come after the $n_i$-th step.

Now we create a new state $\perp$ whose only outgoing transition is a self
loop with reward $-1$.
We transform $\mdp'$ into $\mdp''$ by re-directing all transitions in
$\text{Bad}_{i}$ to the new target state $\perp$ for every $i$.
Notice that any run visiting $\perp$ must be losing for $\liminfppobj$ due to
the negative reward on the self loop, but it must also be losing for $\transience$ because of the self loop.

We now show that the change from $\mdp'$ to $\mdp''$
has decreased the value of $s_0$ for $\liminfppobj$ by at most $2\eps$, i.e.,
\begin{equation}\label{eq:fblose2eps}
\valueof{\mdp'',\liminfppobj}{s_0} \ge \valueof{\mdp',\liminfppobj}{s_0} - 2\eps.
\end{equation}
\cref{eq:fblose2eps} follows from the following steps.
\begin{align*}
\valueof{\mdp'',\liminfppobj}{s_0} 
& \ge \probm_{\mdp'',s_0,\sigma} \left( \bigcap_{i \in \N} \text{Safety}^{n_{i}}_{i}  \right) \\
&
= \probm_{\mdp',s_0,\sigma} \left( \bigcap_{i \in \N} \text{Safety}^{n_{i}}_{i}  \right)
& \ \text{by def. of $\mdp''$}\\
& \ge \valueof{\mdp',\liminfppobj}{s_0} - 2\eps & \ \text{by \cref{claim:fblose2eps}}
\end{align*}

In the next step
(proof in \cref{app:upper-fb})
we argue that under \emph{every} strategy
$\sigma''$ from $s_0$ in $\mdp''$ the attainment for $\liminfppobj$ and
$\transience$ coincide, i.e.,
\begin{restatable}{claim}{eqliminfpptransience}\label{eqliminfpptransience}
\[
\forall \sigma''.\, \probm_{\mdp'',s_0,\sigma''}(\liminfppobj) = \probm_{\mdp'',s_0,\sigma''}(\transience).
\]
\end{restatable}

By \cref{epstransience}, there exists a uniformly $\eps$-optimal MD
strategy $\widehat{\sigma}$ from $s_0$ for $\transience$ in $\mdp''$, i.e.,
\begin{equation}\label{eq:transience-eps}
\probm_{\mdp'',s_0,\hat{\sigma}}(\transience) \ge \valueof{\mdp'',\transience}{s_0}
- \eps.
\end{equation}
We construct an MD strategy $\sigma^{*}$ in $\mdp$ which plays
like $\sigma_{\text{safe}}$ in $S_{\text{safe}}$ and plays like $\widehat{\sigma}$ everywhere else.
\begin{align*}
\probm_{\mdp,\state_0,\zstrat^{*}}(\liminfppobj) 
&=  \probm_{\mdp',\state_0,\hat{\zstrat}}(\liminfppobj) & \text{def. of $\zstrat^{*}$ and $\sigma_{\text{safe}}$}\\
& \ge \probm_{\mdp'',\state_0,\hat{\zstrat}}(\liminfppobj) & \text{new losing sink in $\mdp''$}\\
& = \probm_{\mdp'',\state_0,\hat{\zstrat}}(\transience)  & \text{by \cref{eqliminfpptransience}}\\
& \ge \valueof{\mdp'',\transience}{s_0} - \eps  & \text{by \eqref{eq:transience-eps}}\\
& = \valueof{\mdp'',\liminfppobj}{s_0} - \eps   & \text{by \cref{eqliminfpptransience}}\\
& \ge \valueof{\mdp',\liminfppobj}{s_0} - 2\eps -\eps & \text{by \eqref{eq:fblose2eps}}\\
& = \valueof{\mdp,\liminfppobj}{s_0} - 3\eps & \text{by \eqref{eq:valuesmmprime}}
\end{align*}
Hence $\sigma^{*}$ is a $3\eps$-optimal MD strategy for $\liminfppobj$ from
$s_0$ in $\mdp$ as required.
\end{proof}

\begin{corollary}\label{fintpepsupper}
Given a finitely branching MDP $\mdp$, there exist $\eps$-optimal strategies
for $\liminftpobj$ which use just a reward counter.
\end{corollary}
\begin{proof}
By \cref{finpointpayoff} and \cref{totaltopoint}.
\end{proof}

\begin{corollary}\label{finoptupper}
Given a finitely branching MDP $\mdp$ and initial state $s_{0}$, optimal strategies, where they exist, 
\begin{itemize}
\item for $\liminfppobj$ can be chosen MD.
\item for $\liminftpobj$ can be chosen with just a reward counter.
\end{itemize}
\end{corollary}
\begin{proof}
Since $\liminfppobj$ is tail, the first claim follows
from \cref{finpointpayoff} and \cref{epsilontooptimal}.

Towards the second claim, we place ourselves in $R(\mdp)$ where
$\liminftpobj$ is tail. Moreover, in $R(\mdp)$ the objectives
$\liminftpobj$ and $\liminfppobj$ coincide.
Thus we can apply \cref{finpointpayoff} to obtain $\eps$-optimal MD
strategies for $\liminftpobj$ from every state of $R(\mdp)$.
From \cref{epsilontooptimal} we obtain a single MD
strategy that is optimal from every state of $R(\mdp)$ that has an optimal
strategy. By \cref{totaltopoint} we can translate this MD strategy on $R(\mdp)$ back to
a strategy on $\mdp$ with just a reward counter.
\end{proof}

\subsection{Infinitely Branching Case}\label{subsec:upper-ib}

For infinitely branching MDPs, $\eps$-optimal strategies
for $\liminfppobj$ require more memory than in the finitely branching case.
However, the proofs are similar to those in \cref{subsec:upper-fb}
and can be found in \cref{app:upper-ib}.

% In the following theorem we show
% how to obtain $\eps$-optimal strategies
% for $\liminfppobj$ by reduction to a safety objective
% that approximates it.

\begin{restatable}{theorem}{thminfpointpayoff}\label{infpointpayoff}
Consider an MDP $\mdp$ with initial state $s_{0}$ and a $\liminfppobj$
objective. For every $\eps >0$ there exist
\begin{itemize}
\item $\eps$-optimal MD strategies in $S(\mdp)$.
\item $\eps$-optimal deterministic Markov strategies in $\mdp$.
\end{itemize}
\end{restatable}
\newcommand{\thminfpointpayoffproof}{
\begin{proof}
Let $\eps >0$.
We work in $S(\mdp)$ by encoding the step counter into the states of $\mdp$.
Thus $S(\mdp)$ is an acyclic MDP with implicit step counter and corresponding
initial state $s_0' = (s_0,0)$.

We consider a general (not necessarily MD) $\eps$-optimal strategy $\sigma$
for $\liminfppobj$
from $s_0'$ on $S(\mdp)$, i.e.,
\begin{equation}\label{eq:sigma-eps-opt}
\probm_{S(\mdp),s_0',\sigma}(\liminfppobj) \ge \valueof{S(\mdp),\liminfppobj}{s_0'} - \eps.
\end{equation}
Define the safety objective $\text{Safety}_{i}$ which is the objective of
never seeing any point reward $< -2^{-i}$.
This then allows us to characterize $\liminfppobj$ in terms of safety objectives.
\begin{equation} \label{eq:liminfppissafety}
\liminfppobj = \bigcap_{i \in \mathbb{N}} \eventually(\text{Safety}_{i})
\end{equation}

Now we define the safety objective $\text{Safety}_{i}^{k} \eqdef \eventually^{\leq k}( \text{Safety}_{i} )$ to attain $\text{Safety}_{i}$ within at most $k$ steps. This allows us to write 
\begin{equation} \label{eq:safetyisunion}
\eventually(\text{Safety}_{i}) = \bigcup_{k \in \mathbb{N}} \text{Safety}_{i}^{k}.
\end{equation}
By continuity of measures from above we get 
\begin{align*}
0 & = \mathcal{P}_{S(\mdp),s_0',\sigma} \left( \eventually(\text{Safety}_{i}) \cap \bigcap_{k \in \N} \overline{\text{Safety}^{k}_{i}}\right)\\
  & = \lim_{k \to \infty} \mathcal{P}_{S(\mdp),s_0',\sigma} \left( \eventually(\text{Safety}_{i}) \cap \overline{\text{Safety}^{k}_{i}}\right).
\end{align*}
Hence for every $i \in \mathbb{N}$ and
$\eps_{i} \defeq \eps \cdot 2^{-i}$
there exists $n_{i}$ such that
\begin{equation}\label{eq:ni}
\mathcal{P}_{S(\mdp),s_0',\sigma} \left( \eventually(\text{Safety}_{i}) \cap
\overline{\text{Safety}^{n_{i}}_{i}} \right) \leq \eps_{i}.
\end{equation}

Now we can show the following claim. % (proof in \cref{app:upper}).
\begin{restatable}{claim}{claimlosetwoeps}\label{claim:lose2eps}
\[
\probm_{S(\mdp),s_0',\sigma} \left( \bigcap_{i \in \N} \text{Safety}^{n_{i}}_{i}  \right) 
\geq
\valueof{S(\mdp),\liminfppobj}{s_0'} - 2 \eps.
\]
\end{restatable}

\newcommand{\claimlosetwoepsproof}{
\begin{proof}
\begin{align*}
& \mathcal{P}_{S(\mdp),s_0',\sigma} \left( \bigcap_{i \in \N} \text{Safety}^{n_{i}}_{i} \right) 
\\
& \geq
\mathcal{P}_{S(\mdp),s_0',\sigma} \left(
\bigcap_{k \in \N}  \eventually(\text{Safety}_{k}) \cap \bigcap_{i \in \N} \text{Safety}^{n_{i}}_{i} 
\right) \\
& =  \mathcal{P}_{S(\mdp),s_0',\sigma} \left( \left( \bigcap_{k \in \N}  \eventually(\text{Safety}_{k}) \cap \bigcap_{i \in \N} \text{Safety}^{n_{i}}_{i} \right)
\cup \left(\overline{\bigcap_{k \in \N}  \eventually(\text{Safety}_{k})} \cap \bigcap_{k \in \N}  \eventually(\text{Safety}_{k})\right)
\right) \\
& = \mathcal{P}_{S(\mdp),s_0',\sigma} \left(  \bigcap_{k \in \N}  \eventually(\text{Safety}_{k}) \cap 
\left(\bigcap_{i \in \N} \text{Safety}^{n_{i}}_{i} \cup \overline{\bigcap_{k \in \N}  \eventually(\text{Safety}_{k})} \right)
\right) \\
& = 1 - \mathcal{P}_{S(\mdp),s_0',\sigma} \left(  \overline{ \bigcap_{k \in \N}  \eventually(\text{Safety}_{k})} \cup 
\left(\overline{\bigcap_{i \in \N} \text{Safety}^{n_{i}}_{i}} \cap \bigcap_{k \in \N}  \eventually(\text{Safety}_{k})\right)
\right) \\
& \geq 1 - \mathcal{P}_{S(\mdp),s_0',\sigma} \left(  \overline{ \bigcap_{k \in \N}  \eventually(\text{Safety}_{k})} \right)
 -
\mathcal{P}_{S(\mdp),s_0',\sigma} \left(\overline{\bigcap_{i \in \N} \text{Safety}^{n_{i}}_{i}} \cap \bigcap_{k \in \N}  \eventually(\text{Safety}_{k})\right) \\
& = \mathcal{P}_{S(\mdp),s_0',\sigma} \left(\liminfppobj\right)
 -
\mathcal{P}_{S(\mdp),s_0',\sigma} \left(\overline{\bigcap_{i \in \N} \text{Safety}^{n_{i}}_{i}} \cap \bigcap_{k \in \N}  \eventually(\text{Safety}_{k})\right)
& \ \text{by \eqref{eq:liminfppissafety}}\\
& \geq \valueof{S(\mdp),\liminfppobj}{s_0'} - \eps -
\mathcal{P}_{S(\mdp),s_0',\sigma} \left(\bigcup_{i \in \N} \overline{\text{Safety}^{n_{i}}_{i}} \cap \bigcap_{k \in \N}  \eventually(\text{Safety}_{k})\right)
& \ \text{by \eqref{eq:sigma-eps-opt}}\\
& \geq \valueof{S(\mdp),\liminfppobj}{s_0'} - \eps
 - \sum_{i \in \N} \mathcal{P}_{S(\mdp),s_0',\sigma} \left( \overline{\text{Safety}^{n_{i}}_{i}} \cap \bigcap_{k \in \N} \eventually(\text{Safety}_{k})
\right) \\
& \geq \valueof{S(\mdp),\liminfppobj}{s_0'} - \eps - \sum_{i \in \N} \eps_{i}
& \ \text{by \eqref{eq:ni}}\\
& = \valueof{S(\mdp),\liminfppobj}{s_0'} - 2 \eps
\end{align*}
\end{proof}
}
\claimlosetwoepsproof

Let $\formula \eqdef \bigcap_{i \in \N} \text{Safety}^{n_{i}}_{i} \subseteq \liminfppobj$.
It follows from \cref{claim:lose2eps} that
\begin{equation}\label{eq:lose2eps}
 \valueof{S(\mdp),\formula}{s_0'} \ge
 \valueof{S(\mdp),\liminfppobj}{s_0'} - 2\eps.
\end{equation}
The objective $\formula$ 
is a safety objective on $S(\mdp)$. Therefore, since $S(\mdp)$ is acyclic,
we can apply \cref{acyclicsafety} to obtain a uniformly $\eps$-optimal MD
strategy $\sigma'$ for $\formula$. Thus
\begin{align*}
& \probm_{S(\mdp),s_0',\sigma'}(\liminfppobj) \\
& \ge \probm_{S(\mdp),s_0',\sigma'}(\formula)   & \ \text{set inclusion} \\
& \ge \valueof{S(\mdp),\formula}{s_0'} - \eps  & \ \text{$\sigma'$ is $\eps$-opt.}\\
& \ge \valueof{S(\mdp),\liminfppobj}{s_0'} - 3\eps. & \ \text{by \eqref{eq:lose2eps}} 
\end{align*}
Thus $\sigma'$ is a $3\eps$-optimal MD strategy for $\liminfppobj$ in $S(\mdp)$.

By \cref{steptomarkov} this then yields a $3\eps$-optimal Markov strategy for
$\liminfppobj$ from $s_{0}$ in $\mdp$, since runs in $\mdp$ and $S(\mdp)$
coincide wrt.\ $\liminfppobj$.
\end{proof}
}

\begin{restatable}{corollary}{cormpepsupper}\label{mpepsupper}
Given an MDP $\mdp$ and initial state $s_{0}$, there exist $\eps$-optimal strategies $\sigma$ for $\liminfmpobj$ which use just a step counter and a reward counter.
\end{restatable}
\newcommand{\cormpepsupperproof}{
\begin{proof}
We consider the encoded system $A(\mdp)$ in which both step counter and reward counter are implicit in the state. Recall that the partial mean payoffs in $\mdp$ correspond exactly to point rewards in $A(\mdp)$. Since $A(\mdp)$ has an encoded step counter, \cref{infpointpayoff} gives us $\eps$-optimal MD strategies for $\liminfppobj$ in $A(\mdp)$. \cref{meantopoint} allows us to translate these strategies back to $\mdp$ with a memory overhead of just a reward counter and a step counter as required.
\end{proof}
}

\begin{restatable}{corollary}{corinftpepsupper}\label{inftpepsupper}
Given an MDP $\mdp$ with initial state $s_{0}$,
\begin{itemize}
\item there exist $\eps$-optimal MD strategies for $\liminftpobj$ in $S(R(\mdp))$,
\item there exist $\eps$-optimal strategies for $\liminftpobj$ which use a step counter and a reward counter.
\end{itemize}
\end{restatable}
\newcommand{\corinftpepsupperproof}{
\begin{proof}
We consider the encoded system $R(\mdp)$ in which the reward counter is implicit in the state. Recall that total rewards in $\mdp$ correspond exactly to point rewards in $R(\mdp)$. We then apply \cref{infpointpayoff} to $R(\mdp)$ to obtain  $\eps$-optimal MD strategies for $\liminfppobj$ in $S(R(\mdp))$. 
\cref{steptomarkov} allows us to translate these MD strategies back to $R(\mdp)$ with a memory overhead of just a step counter. Then we apply
\cref{totaltopoint} to translate these Markov strategies back to $\mdp$ with a memory overhead of just a reward counter. Hence $\eps$-optimal strategies for $\liminftpobj$ in $\mdp$ just use a step counter and a reward counter as required.
\end{proof}
}

% \begin{remark}
% While $\eps$-optimal strategies for mean payoff and total payoff
% (in infinitely branching MDPs) have the same memory requirements, the step counter and the reward counter do not arise in the same way.
% Both the step counter and reward counter used in $\eps$-optimal strategies for mean payoff  arise from the construction of $A(\mdp)$. However, in the case for total payoff, only the reward counter arises from the construction of $R(\mdp)$. The step counter on the other hand arises from the Markov strategy needed for point payoff in $R(\mdp)$. 
% \end{remark}

\begin{restatable}{corollary}{corinfoptupper}\label{infoptupper}
Given an MDP $\mdp$ and initial state $s_{0}$, optimal strategies, where they exist, 
\begin{itemize}
\item for $\liminfppobj$ can be chosen with just a step counter.
\item for $\liminfmpobj$ and $\liminftpobj$ can be chosen with just a reward counter and a step counter.
\end{itemize}
\end{restatable}
\newcommand{\corinfoptupperproof}{
\begin{proof}
To obtain the result for $\liminfppobj$, we work in $S(\mdp)$ and we
apply \cref{infpointpayoff} to obtain $\eps$-optimal MD strategies from every
state of $S(\mdp)$.
Since $\liminfppobj$ is a tail objective, \cref{epsilontooptimal} yields an MD
strategy that is optimal from every state of $S(\mdp)$ that has an optimal
strategy. By \cref{steptomarkov} we can translate
this MD strategy on $S(\mdp)$ back to a Markov strategy in $\mdp$,
which is optimal for $\liminfppobj$ from $s_{0}$ (provided that $s_0$ admits
any optimal strategy at all).

Consider the case for $\liminfmpobj$. First we place ourselves in $A(\mdp)$
and apply \cref{infpointpayoff} to obtain  $\eps$-optimal MD strategies from
every state of $A(\mdp)$. From \cref{epsilontooptimal} we obtain a single MD
strategy that is optimal from every state of $A(\mdp)$ that has an optimal
strategy.
By \cref{meantopoint} we can translate this MD strategy on $A(\mdp)$ back to
a strategy on $\mdp$ with a step counter and a reward counter.
Provided that $s_{0}$ admits any optimal strategy at all, we obtain
an optimal strategy for $\liminfmpobj$ from $s_{0}$ that uses only a step counter and a reward counter.

The case for $\liminftpobj$ is similar. We place ourselves in
$S(R(\mdp))$ and apply \cref{inftpepsupper} to obtain $\eps$-optimal MD
strategies for $\liminftpobj$ from every state of $S(R(\mdp))$.
While $\liminftpobj$ is not tail in $\mdp$, it is tail in $S(R(\mdp))$,
and thus we can apply \cref{epsilontooptimal} to obtain a single MD
strategy that is optimal from every state of $S(R(\mdp))$ that has an optimal
strategy. The result then follows from \cref{totaltopoint} and \cref{steptomarkov}.
\end{proof}
}

\bigskip
\section{Conclusion and Outlook}\label{sec:conclusion}
We have established matching lower and upper bounds on the strategy complexity
of $\liminf$ threshold objectives for point, total and mean payoff on countably infinite
MDPs; cf.~\cref{table:allresults}.

The upper bounds hold not only for integer transition rewards, but also
for rationals or reals, provided that the reward counter (in those cases where one
is required) is of the same type.
The lower bounds hold even for integer transition rewards, since all our
counterexamples are of this form.

Directions for future work include the corresponding questions for $\limsup$
threshold objectives. While the $\liminf$ point payoff objective generalizes co-B\"uchi
(see \cref{sec:prelim}), the $\limsup$ point payoff objective generalizes
B\"uchi. Thus the lower bounds for $\limsup$ point payoff are at least as high as
the lower bounds for B\"uchi objectives \cite{KMST:ICALP2019,KMST2020c}.

\newpage
%\bibliography{refs,journals,conferences}

\newpage
\appendix
\section{Introduction to Strategy Complexity}\label{app-def}
\noindent{\bf A simple example.}

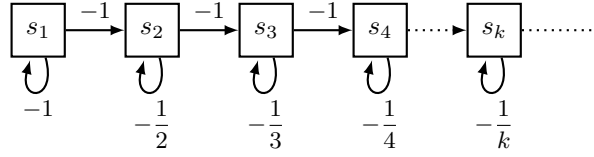
\begin{figure}[htbp]
\begin{center}
    \begin{tikzpicture}
    
    % the gadget nodes
    \node[draw, minimum height=0.7cm, minimum width=0.7cm] (S1) at (1,0) {$s_{1}$};
    \node[draw, minimum height=0.7cm, minimum width=0.7cm] (S2) at (2.5,0) {$s_{2}$};
    \node[draw, minimum height=0.7cm, minimum width=0.7cm] (S3) at (4,0) {$s_{3}$};
    \node[draw, minimum height=0.7cm, minimum width=0.7cm] (S4) at (5.5,0) {$s_{4}$};
    
    \node[draw, minimum height=0.7cm, minimum width=0.7cm] (SK) at (7,0) {$s_{k}$};
    \node (SI1) at (8.5,0) {};

    % drawing the lines
    
    \draw[->,>=latex] (S1) -- (S2) node[above, midway]{\small $-1$};
    \draw[->,>=latex] (S2) -- (S3) node[above, midway]{\small $-1$};
    \draw[->,>=latex] (S3) -- (S4) node[above, midway]{\small $-1$};
    \draw[->,>=latex, dotted, thick] (S4) -- (SK) node[above, midway]{};
    \draw[dotted, thick] (SK) -- (SI1);
    
	\path (S1) edge[->,>=latex, loop below] node[below] {\small $-1$} (S1);
	\path (S2) edge[->,>=latex, loop below] node[below] {\small $-\dfrac{1}{2}$} (S2);
	\path (S3) edge[->,>=latex, loop below] node[below] {\small $-\dfrac{1}{3}$} (S3);
	\path (S4) edge[->,>=latex, loop below] node[below] {\small $-\dfrac{1}{4}$} (S4);
	\path (SK) edge[->,>=latex, loop below] node[below] {\small $-\dfrac{1}{k}$} (SK);

    \end{tikzpicture}
    \caption{Adapted from \cite[Example 8.10.2]{Puterman:book}. While there is
      no optimal MD (memoryless deterministic) strategy, the following
      strategy is optimal for lim inf/lim sup mean payoff:
      Loop $\exp(\exp(k))$ many
      % $2^{2^k}$
      times
      in state $s_k$ for all $k$.    
      In this particular example, this can be implemented with either
      just a step counter or just a reward counter, but
      in general both are needed; cf.~\cref{table:allresults}.}
    \label{putermanexample}
\end{center}    
    \end{figure}

\noindent{\bf Memory and strategies.}

We formalize the amount of \emph{memory} needed to implement strategies.
Let $\memory$ be a countable set of memory modes, and let $\tau: \memory\times \states \to \dist(\memory\times \states)$ be a function that meets the
following two conditions: for all modes $\memconf \in \memory$,
\begin{itemize}
	\item for all controlled states~$\state\in \zstates$, 
	the distribution $\tau(\memconf,\state)$ is   over 
	$\memory \times \{\state'\mid \state \transition{} \state'\}$.
	\item for all random states~$\state \in \rstates$,
        and $\state' \in \states$, we have $\sum_{\memconf'\in \memory} \tau(\memconf,\state)(\memconf',\state')=P(\state)(\state')$.
\end{itemize}

The function~$\tau$ together with an initial memory mode~$\memconf_0$
induce a strategy~$\zstrat_{\tau}$
% :\states^*\zstates \to \dist(S)$
as follows.
Consider the Markov chain with the set~$\memory \times \states$ of states
and the  probability function~$\tau$.
A sequence $\rho=s_0 \cdots s_i$ corresponds to a set 
$H(\rho)=\{(\memconf_0,s_0) \cdots (\memconf_i,s_i) \mid \memconf_0,\ldots, \memconf_i\in \memory\}$
of runs in this Markov chain.
Each $\rho s \in \state_0 \states^{*} \zstates$
induces a 
probability distribution~$\mu_{\rho \state}\in \dist(\memory)$, 
%the probability $\mu_{\rho \state}(\memconf)$ is 
the  probability of   being in  state~$(\memconf,s)$
conditioned on having  taken some partial run
from~$H(\rho s)$.
We define~$\zstrat_{\tau}$ such that
$\zstrat_{\tau}(\rho \state)(\state')=\sum_{\memconf,\memconf'\in \memory} \mu_{\rho \state}(\memconf) \tau(\memconf,\state)(\memconf',\state') $
for all $\rho \state\in \states^{*} \zstates$ and all $\state' \in \states$.

We say that a strategy $\zstrat$ can be \emph{implemented} with
memory~$\memory$ if there exist~$\memconf_0 \in \memory$
and $\tau$ such that  $\zstrat_{\tau}=\zstrat$.

% \smallskip
% \noindent{\bf Upper and lower bounds on strategy complexity.}

\newpage
\section{Missing proofs from Section~\ref{sec:liminfreward}}\label{sec:appreward}
\lemwelldefined*
\lemwelldefinedproof

\begin{proposition}\label{prop:product-sum}
Given an infinite sequence of real numbers $a_n$ with $0 \le a_n \le 1$, we
have
\[
\prod_{n=1}^\infty (1-a_n) > 0 \quad\Leftrightarrow\quad \sum_{n=1}^\infty a_n
< \infty.
\]
\end{proposition}
\begin{proof}
In the case where $a_n$ does not converge to zero, the property is trivial.
In the case where $a_n \rightarrow 0$, it is shown
by taking the logarithm of the product and using the limit comparison test as follows.

Taking the logarithm of the product gives the series
\[
\sum_{n=1}^{\infty} \ln(1 - a_n)
\]
whose convergence (to a finite number $\le 0$) is equivalent to the positivity of the product.
It is also equivalent to the convergence (to a number $\ge 0$) of its negation
$\sum_{n=1}^{\infty} -\ln(1 - a_n)$.
But observe that (by L'H\^{o}pital's rule)
\[
\lim_{x \rightarrow 0} \frac{-\ln(1-x)}{x} = 1.
\]
Since $a_n \rightarrow 0$ we have
\[
\lim_{n \rightarrow \infty} \frac{-\ln(1-a_n)}{a_n} = 1.
\]
By the limit comparison test, the series
$\sum_{n=1}^{\infty} -\ln(1 - a_n)$ converges if and only if the series
$\sum_{n=1}^\infty a_n$ converges.
\end{proof}

\begin{proposition}\label{prop:tail-product}
Given an infinite sequence of real numbers $a_n$ with $0 \le a_n \le 1$,
\[
\prod_{n=1}^\infty a_n > 0 \quad\Rightarrow\quad \forall\eps>0\,\exists N.\, \prod_{n=N}^\infty a_n \ge (1-\eps).
\]
\end{proposition}
\begin{proof}
  Since $\prod_{n=1}^\infty a_n > 0$, by taking the logarithm we obtain
  $\sum_{n=1}^\infty \ln(a_n) > -\infty$.
  Thus for every $\delta>0$ there exists an $N$
  s.t.\ $\sum_{n=N}^\infty \ln(a_n) \ge -\delta$.
  By exponentiation we obtain
  $\prod_{n=N}^\infty a_n \ge \exp(-\delta)$.
  By picking $\delta = -\ln(1-\eps)$ the result follows.
\end{proof}

\lemmainfwin*
\lemmainfwinproof

\begin{restatable}{lemma}{lemmaalglose}\label{alglose}
For any sequence $\{\alpha_{n}\}$, where $\alpha_{n} \in [0,1]$ for all $n$, and any functions $i(n), j(n) : \mathbb{N} \to \N $ with $i(n), j(n) \in \{0, 1, ..., k(n)-1\}, i(n) < j(n)$ for all $n$, the following sum diverges:

\begin{equation}\label{ijlosingsum}
\sum_{n=k^{-1}(2)}^{\infty} \Big( \delta_{j(n)}(n)(\alpha_{n} \varepsilon_{j(n)}(n) + (1- \alpha_{n}) \varepsilon_{i(n)}(n)) + \delta_{i(n)}(n)(\alpha_{n} + (1- \alpha_{n}) \varepsilon_{i(n)}(n)) \Big).
\end{equation}
\end{restatable}

\newcommand{\lemmaalgloseproof}{
\begin{proof}
We can narrow our focus by noticing that 
\begin{align*}
& \sum_{n=k^{-1}(2)}^{\infty} \Big( \delta_{j(n)}(n)(\alpha_{n} \varepsilon_{j(n)}(n) + (1- \alpha_{n}) \varepsilon_{i(n)}(n)) + \delta_{i(n)}(n)(\alpha_{n} + (1- \alpha_{n}) \varepsilon_{i(n)}(n)) \Big) \\
& = \sum_{n=k^{-1}(2)}^{\infty} \alpha_{n} \delta_{j(n)}(n) \varepsilon_{j(n)}(n) + (1- \alpha_{n}) \delta_{i(n)}\varepsilon_{i(n)}(n) \qquad \text{Convergent by def. of $\delta_{i}(n), \varepsilon_{i}(n)$} \\ 
& + \sum_{n=k^{-1}(2)}^{\infty}(1- \alpha_{n}) \delta_{j(n)} \varepsilon_{i(n)}(n) + \alpha_{n} \delta_{i(n)}(n)
\end{align*}

Hence the divergence of \eqref{ijlosingsum} depends only on the divergence of $\sum_{n=k^{-1}(2)}^{\infty}(1- \alpha_{n}) \delta_{j(n)} \varepsilon_{i(n)}(n) + \alpha_{n} \delta_{i(n)}(n)$.
No matter how the sequence $\{ \alpha_{n} \}$ behaves, for every $n$ we have that either $\alpha_{n} \geq 1/2$ or $1- \alpha_{n} \geq 1/2$. Hence for every $n$ it is the case that 
\begin{align*}
(1-\alpha_{n}) \delta_{j(n)}(n) \varepsilon_{i(n)}(n) +
  \alpha_{n}\delta_{i(n)}(n) \,\geq\, & \dfrac{1}{2} \delta_{j(n)}(n) \varepsilon_{i(n)}(n) \\
\text{or } & \\ 
\,\geq\, & \dfrac{1}{2} \delta_{i(n)}(n)
\end{align*}

Define the function $f$ as follows:
 \[
    f(n)=\left\{
                \begin{array}{ll}
                \dfrac{1}{2} \delta_{i(n)}(n) \text{ if $\alpha_{n} \geq 1/2$} \\
                  \\
                  \dfrac{1}{2} \delta_{j(n)}(n) \varepsilon_{i(n)}(n) \text{ otherwise}
                \end{array}
              \right.
  \]

Hence no matter how $\{ \alpha_{n} \}$ behaves, we have that 
$$
\sum_{n= k^{-1}(2)}^{\infty} \Big( \delta_{j(n)}(n)(\alpha_{n} \varepsilon_{j(n)}(n) + (1- \alpha_{n}) \varepsilon_{i(n)}(n)) + \delta_{i(n)}(n)(\alpha_{n} + (1- \alpha_{n}) \varepsilon_{i(n)}(n)) \Big) \geq \sum_{n= k^{-1}(2)}^{\infty} f(n).
$$
We know that both 
$\sum_{n= k^{-1}(2)}^{\infty} \dfrac{1}{2} \delta_{j(n)}(n) \varepsilon_{i(n)}(n)$
and
$\sum_{n= k^{-1}(2)}^{\infty} \dfrac{1}{2} \delta_{i(n)}(n)$
diverge for all $i(n),j(n) \in \{0,1, ... , k(n)-1\}$,
$i(n) < j(n)$,
as shown in \cref{claim:divergence}. 

Thus 
$\sum_{n= k^{-1}(2)}^{\infty} \dfrac{1}{2} \delta_{j(n)}(n) \varepsilon_{i(n)}(n)$ 
and 
$\sum_{n=k^{-1}(2)}^{\infty} \dfrac{1}{2} \delta_{i(n)}(n)$ must also diverge no matter how $i(n)$ and $j(n)$ behave. As a result  it must be the case that
$\sum_{n=k^{-1}(2)}^{\infty} f(n)$ diverges.
Hence \eqref{ijlosingsum} must be divergent as desired as $i(n)$ and $j(n)$ vary for $n \geq k^{-1}(2)$.
\end{proof}
}
\lemmaalgloseproof

\begin{claim}\label{claim:divergence}
The sum $\sum_{n= k^{-1}(2)}^{\infty} \dfrac{1}{2} \delta_{j(n)}(n) \varepsilon_{i(n)}(n)$ diverges for all $i(n),j(n) \in \{0,1, ... , k(n)-1\}$ with $i(n) < j(n)$.
\end{claim}

\begin{proof}
This result is not immediate because the range of values the indexing functions $i(n)$ and $j(n)$ can take grows with $k(n)$ as $n$ increases.

Under the assumption that $i(n) < j(n)$ we have that $\delta_{j(n)}(n) \eps_{i(n)}(n) \geq \delta_{j(n)}(n) \eps_{j(n)-1}(n) \geq \delta_{k(n)-1}(n) \eps_{k(n)-2}(n) = \eps_{k(n)-1}(n)$.
Thus it suffices to show that $\sum_{n=k^{-1}(2)}^{\infty} \eps_{k(n)-1}(n)$ diverges:

%We can further limit the scope of our proof by observing that 
%$\delta_{k(n)-1}(n) \eps_{k(n)-2}(n) \leq \delta_{j(n)}(n) \eps_{j(n)-1}(n)$ for all $j(n) \in \{1, ..., k(n)-1 \}$.
%For ease of notation, notice that $\delta_{k(n)-1}(n) \eps_{k(n)-2}(n) = \eps_{k(n)-1}(n)$, 
%so what we are trying to show is in fact that 
%$\sum_{n=k^{-1}(2)}^{\infty} \eps_{k(n)-1}(n)$ diverges.

\begin{align*}
\sum_{n=k^{-1}(2)}^{\infty} \eps_{k(n)-1}(n) & = \sum^{\infty}_{a=2} \quad \sum_{n= k^{-1}(a)}^{k^{-1}(a+1)-1} \eps_{a-1}(n) & \text{splitting the sum up}\\
& = \sum^{\infty}_{a=2} \quad \sum_{n= h(a)}^{h(a+1)-1} \eps_{a-1}(n) & \text{$k(n) = h^{-1}(n)$}\\
& \geq \sum^{\infty}_{a=2} 1 & \text{definition of $h(n)$}\\
\end{align*}

Note that the definition of $h(i)$ says exactly that a block of the form $\sum_{n= h(a)}^{h(a+1)-1} \eps_{a-1}(n)$ is at least $1$.
Hence $\sum_{n= k^{-1}(2)}^{\infty} \dfrac{1}{2} \delta_{j(n)}(n) \varepsilon_{i(n)}(n)$ diverges as required.

\end{proof}

\bigskip
\lemmainflose*
\begin{proof}
Let $\sigma$ be some FR strategy with $k$ memory modes.
Our MDP consists of a linear sequence of gadgets (\cref{infinitegadget}) and
is in particular acyclic.
The $n$-th gadget is entered at state $s_n$ and takes 4 steps.
Locally in the $n$-th gadget there are 3 possible scenarios:
\begin{description}
\item[(1)]
The random transition picks some branch $i$ at $s_n$ and the strategy then
picks a branch $j > i$ at $c_n$.

By the definition of the payoffs (multiples of $m_n$; cf.~\cref{def:kn}),
this means that we see a mean payoff $\le -1$, regardless of events in past
gadgets.
This is because the numbers $m_n$ grow so quickly
with $n$ that even the combined maximal possible rewards of all past gadgets
are so small in comparison that they do not matter for the outcome
in the $n$-th gadget, i.e.,
rewards from past gadgets cannot help to avoid seeing a mean payoff $\le -1$
in the above scenario.
\item[(2)]
We reach the losing sink $\bot$ (and thus will keep seeing a mean payoff $\le -1$
forever). This happens with probability $\eps_j(n)$ if the strategy picks
some branch $j$ at $c_n$, regardless of past events.
\item[(3)]
All other cases.
\end{description}
As explained above, due to the definition of the rewards (\cref{def:kn}),
events in past gadgets do not make the difference between (1),(2),(3) in the
current gadget. 
It just depends on the choices of the strategy $\sigma$ in the current gadget.

Let ${\it Bad}_n$ be the event of seeing either of the two unfavorable outcomes (1) or (2) in the $n$-th
gadget. Let $p_n$ be the probability of ${\it Bad}_n$ under strategy $\sigma$.
Since $\sigma$ has memory, the probabilities $p_n$ are not necessarily
independent.
However, we show \emph{lower bounds} $e_n \le p_n$ that hold universally for every
FR strategy $\sigma$ with $\le k$ memory modes and every $n$ such that
$k(n) > k+1$.
The lower bound $e_n$ will hold regardless of the memory mode of $\sigma$ upon
entering the $n$-th gadget.

{\bf\noindent Memory updates.}
First we show that $\sigma$ randomizing its memory update after observing
the random transition from state $s_n$ does \emph{not} help to reduce the
probability of event ${\it Bad}_n$.
I.e., we show that without restriction $\sigma$ can update its memory
deterministically after observing the transition from state $s_n$.

Once in the controlled state $c_n$, the strategy $\sigma$ can base its choice only on the
current state (always $c_n$ in the $n$-th gadget)
and on the current memory mode.
Thus, in state $c_n$, in each memory mode $\memconf$, the strategy has
to pick a distribution $\mathcal{D}^{c_n}_{\memconf}$
over the available transitions from $c_n$.
By the finiteness
of the number of memory modes of $\sigma$
(just $\le k$ by our assumption above),
for each possible
reward level $x$ (obtained in the step from the preceding random transition)
there is a best memory mode $\memconf(x)$ such that $\mathcal{D}^{c_n}_{\memconf(x)}$ is optimal
(in the sense of minimizing the probability of event ${\it Bad}_n$)
for that particular reward level $x$.
(In case of a tie, just use an arbitrary tie break, e.g.,
some pre-defined linear order on the memory modes.)

Therefore, upon witnessing a reward level $x$ in the random transition from
state $s_n$, the strategy $\sigma$ can minimize the probability of event
${\it Bad}_n$ by \emph{deterministically} setting its memory to $\memconf(x)$.
Thus randomizing its memory update does not help to reduce the probability
of ${\it Bad}_n$, and we may assume without restriction that $\sigma$ updates
its memory deterministically.

% \footnote{
(Note that the above argument only works because it is local to the current
gadget where we have a finite number of decisions (here just one),
we have a finite number of memory modes, and a one-dimensional criterion for
local optimality (minimizing the probability of event ${\it Bad}_n$).
We do \emph{not} claim that randomized memory updates are useless for every strategy
in every MDP and every objective.)
%}

{\bf\noindent The lower bounds $e_n$.}
Now we consider an FR strategy $\sigma$ that without restriction updates its memory
\emph{deterministically} after each random choice (from state $s_n$) in the $n$-th
gadget. It can still randomize its actions, however.

Let $N'$ be the minimal number such that for all $n \ge N'$ we
have $k(n) > k+1$. In particular, this implies $N' \ge k^{-1}(2)$,
and thus we can apply \cref{alglose} later.

Once $n \ge N'$, then by the Pigeonhole Principle there will always be a
memory mode confusing at least two different transitions $i(n),j(n) \neq k(n)$
from state $s_n$ to $c_n$.
Note that this holds regardless of the memory mode of $\sigma$ upon entering
the $n$-th gadget.
(The strategy might confuse many other scenarios, but just one confused pair
$i(n),j(n) \neq k(n)$ is enough for our lower bound.)
Without loss of generality, let $j(n)$ be larger of the two confused transitions, i.e., $i(n) < j(n)$.
Let $i(n)$ and $j(n)$ be two functions taking values in $\{0, 1, ..., k(n)-1\}$ where $i(n) < j(n)$ for all $n$.

Confusing two transitions $i(n)$ and $j(n)$ from $s_n$ to $c_n$
(where without restriction $i(n) < j(n)$), the strategy is in the same
memory mode afterwards. However, it can still randomize its choices in state
$c_n$.
To prove our lower bound on the probability of ${\it Bad}_n$,
it suffices to consider the case where the strategy
only randomizes over the outgoing transitions $i(n)$ and $j(n)$ from state $c_n$.
This is because, by \cref{claim:confusion-simple}, every other
behavior would perform even worse, in the sense of yielding a higher
probability of ${\it Bad}_n$.

That is to say that the strategy picks the higher $j(n)$-th branch with some probability $\alpha_n$
and the lower $i(n)$-th branch with probability $1-\alpha_n$.
(We leave the probabilities $\alpha_n$ unspecified here. Using \cref{alglose},
we'll show that our result holds regardless of their values.)

The local chance of 
the event ${\it Bad}_n$  is then lower bounded by 

$$e_n \eqdef \delta_{j(n)}(n)(\alpha_{n} \varepsilon_{j(n)}(n) + (1- \alpha_{n}) \varepsilon_{i(n)}(n)) + \delta_{i(n)}(n)(\alpha_{n} + (1- \alpha_{n}) \varepsilon_{i(n)}(n)).
$$

The term above just expresses a case distinction.
In the first scenario, the random transition chooses the $j(n)$-th branch
(with probability $\delta_{j(n)}(n)$) and then the strategy
chooses the $j(n)$-th branch with probability $\alpha_n$
and the lower $i(n)$-th branch with probability $1-\alpha_n$, and you obtain the
respective chances of reaching the sink $\bot$.
In the second scenario, the random transition chooses the $i(n)$-th branch
(with probability $\delta_{i(n)}(n)$). If the strategy then
chooses the higher $j(n)$-th branch (with probability $\alpha_n$) then we have
outcome (1), yielding a mean payoff $\le -1$.
If the strategy chooses the $i(n)$-th branch (with probability $1-\alpha_n$)
then we still have a chance of $\varepsilon_{i(n)}(n)$
of reaching the sink.

Since, as shown above, randomized memory updates do
not help to reduce the probability of ${\it Bad}_n$, the lower bound
$e_n$ for deterministic updates carries over to the general case.
Thus, even for general randomized FR strategies $\sigma$ with $k$ memory
modes, the probability of event ${\it Bad}_n$ in the $n$-th gadget
(for $n \ge N'$) is lower bounded by $e_n$,
regardless of the memory mode $\memconf$ upon entering the
gadget and regardless of events in past gadgets.
We write $\sigma[\memconf]$ for the strategy $\sigma$ in memory mode
$\memconf$ and obtain
\begin{equation}\label{eq:local-bad}
\forall n \ge N'.\ \forall \memconf.\ \probm_{\mathcal{M}, \sigma[m], s_{n}}({\it Bad}_n) \ge e_n
\end{equation}

{\bf\noindent The final step.}
Let ${\it Bad} \eqdef \cup_n {\it Bad}_n$.

Since $i(n), j(n) \neq k(n)$ and $N' \ge k^{-1}(2)$,
we apply \cref{alglose} to conclude that the
series $\sum_{n=N'}^{\infty} e_n =
\sum_{n=N'}^{\infty} \delta_{j(n)}(n)(\alpha_{n} \varepsilon_{j(n)}(n) +
(1- \alpha_{n}) \varepsilon_{i(n)}(n)) + \delta_{i(n)}(n)(\alpha_{n} +
(1- \alpha_{n}) \varepsilon_{i(n)}(n))$ 
is divergent, regardless of the behavior of $i(n), j(n)$ or the sequence
$\{\alpha_{n}\}$.

Finally, we obtain
\begin{align*}
& \probm_{\mathcal{M}, \sigma, s_{0}}(\liminfmpobj) \\
& \le \probm_{\mathcal{M}, \sigma, s_{0}}(\eventually\always \neg{\it Bad})
& \text{set inclusion}\\
& = \probm_{\mathcal{M}, \sigma, s_{0}}\left(\bigcup_l\eventually^{\le
l}\always \neg{\it Bad}\right) & \text{def. of $\eventually$}\\
& = \lim_{l \to \infty} \probm_{\mathcal{M}, \sigma, s_{0}}(\eventually^{\le l}\always \neg{\it Bad}) & \text{continuity of measures}\\
& \le \lim_{l \to \infty} \probm_{\mathcal{M}, \sigma, s_{0}}\left(\bigcap_{n \ge l/4} \neg{\it Bad}_n\right) & \text{4 steps per gadget}\\
& \le \lim_{4N' \le l \to \infty} \prod_{n \ge l/4 \ge
N'} (\max_\memconf\,\probm_{\mathcal{M}, \sigma[\memconf], s_{n}}(\neg{\it Bad}_n))
& \begin{tabular}{l}{linear sequence of gadgets, finite memory},\\ and past
events do not help to avoid ${\it Bad}_n$\end{tabular}\\
& \le \lim_{4N' \le l \to \infty} \prod_{n \ge l/4 \ge N'} (1-e_n) & \text{by \eqref{eq:local-bad}}\\
& = \lim_{4N' \le l \to \infty} 0 & \text{divergence of $\sum_{n = N'}^\infty e_n$ and \cref{prop:product-sum}}\\
& = 0
\end{align*}
\end{proof}

\begin{figure}
\begin{center}
    \begin{tikzpicture}
    
    \node[draw,circle, minimum height=1cm] (R) at (0,0) {$s_{n}$};
    \node[draw, minimum height=0.4cm, minimum width=0.4cm] (N) at (5,0) {$c_{n}$};
    \node[draw,circle] (S) at (10,0) {$s_{n+1}$};
    \node (M) at (6,0){};
    
    %Dots off to the left and right
    \node[] (Left) at (-1,0) {};
    \node[] (Right) at (11,0) {};
    \draw [dotted, ultra thick] (Left) -- (R);
    \draw [dotted, ultra thick] (S) -- (Right);

    \node[draw, circle] (Top) at (2.5,1) {};
    \node[draw, circle] (Bot) at (2.5,-1) {};
    
    \node[draw, circle] at (7.5,-1) (B) {};
    \node[draw, circle] (E) at (7.5,1) {};
    %\node[draw, circle] at (7.5,-2) (Deaddown) {$\perp$};
    \node[draw, circle] at (7.5,0) (Deadup) {$\perp$};

    \draw[->,>=latex] (R) edge[bend left=20] node[above, midway, sloped]{$\delta_{j(n)}(n)$} (Top)
                      (Top) edge[bend left=20] node[above, midway]{$+j(n)m_{n}$} (N) 
                      (R) edge[bend right=20] node[below, midway, sloped]{$\delta_{i(n)}(n)$} (Bot)
                      (Bot) edge[bend right=20] node[below, midway]{$+i(n)m_{n}$} (N);

    \draw[->,>=latex] (N) edge[bend left=20] node[above, midway]{$ \alpha_{n}$} (E)
                      (E) edge[bend left=20] node[above, midway]{$-j(n)m_{n}$} (S)
                      (N) edge[bend right=20] node[below, midway]{$1- \alpha_{n}$} (B)
                      (B) edge[bend right=20] node[below, midway]{$-i(n)m_{n}$} (S)
                      (E) edge node[right, midway]{$\varepsilon_{j(n)}(n)$} (Deadup)
                      (B) edge node[right, midway]{$\varepsilon_{i(n)}(n)$} (Deadup);
    
    \end{tikzpicture}
    \caption{When transitions $i(n)$ and $j(n)$ are confused in the player's memory, the player's choice is at least as bad as the reduced play in this simplified gadget.}
    \label{ijcase}
    \end{center}
    \end{figure}
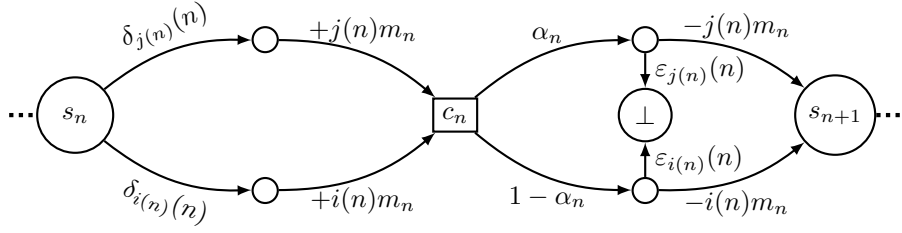
    

\begin{claim}\label{claim:confusion-simple}
Assume that the transitions $i(n)$ and $j(n)$
(with $i(n) < j(n)$) leading to state $c_n$ are confused in the memory of the
strategy. Then we can assume without restriction that the strategy 
only plays transitions $i(n)$ and $j(n)$ with nonzero probability
from state $c_n$, since every other behavior yields a higher probability of
the event ${\it Bad}_n$ (cf.~\cref{ijcase}).
\end{claim}
\begin{proof}
When confusing transitions $i(n)$ and $j(n)$ with $i(n) < j(n)$,
the player's choice of transition from $c_n$ can be broken down into 5 distinct cases.
The player can choose transition $x(n)$ as follows.
\begin{description}
\item[1.] $x(n) = i(n)$
\item[2.] $x(n) = j(n)$
\item[3.] $x(n) > j(n)$
\item[4.] $x(n) < i(n)$
\item[5.] $i(n) < x(n) < j(n)$
\end{description}

Case 1 leads to a probability of ${\it Bad}_n$ of $\delta_{j(n)}(n) \eps_{i(n)}(n) + \delta_{i(n)}(n) \eps_{i(n)}(n)$.

Case 2 leads to a probability of ${\it Bad}_n$ of $\delta_{j(n)}(n) \eps_{j(n)}(n) + \delta_{i(n)}(n)$.

Case 3 leads to a mean payoff $\le -1$ (and thus ${\it Bad}_n$) with probability $1$. 
This is the worst possible case.

Case 4 leads to a probability of ${\it Bad}_n$ of
$\delta_{j(n)}(n) \eps_{x(n)}(n) + \delta_{i(n)}(n) \eps_{x(n)}(n) >
\delta_{j(n)}(n) \eps_{i(n)}(n) + \delta_{i(n)}(n) \eps_{i(n)}(n)$, i.e.,
this is worse than Case 1.

Case 5 leads to a probability of ${\it Bad}_n$ of
$\delta_{j(n)}(n) \eps_{x(n)}(n) + \delta_{i(n)}(n) > \delta_{j(n)}(n) \eps_{j(n)}(n) + \delta_{i(n)}(n)$,
i.e., this is worse than Case 2.

Hence, without restriction we can assume that
only cases 1 and 2 will get played with positive probability,
that is to say that in state $c_n$ the strategy will only randomize over
the outgoing transitions $i(n)$ and $j(n)$.
\end{proof}

\bigskip
\lemmaalmostwin*
\lemmaalmostwinproof

\bigskip
\lemmaalmostlose*
\lemmaalmostloseproof

\newpage
\section{Missing proofs from Section~\ref{sec:liminfstep}}\label{app:step}
In this part we show that a reward counter plus arbitrary finite memory
does not suffice for ($\eps$-)optimal strategies for $\liminfmpobj$
or for infinitely branching $\liminftpobj$/$\liminfppobj$ in countable MDPs.

First we consider $\liminfmpobj$ by presenting an MDP adapted
from \cref{infinitegadget} that has the current total reward implicit in the
state and show that neither $\eps$-optimal nor almost-sure $\liminfmpobj$ can
be achieved by FR strategies (finite memory randomized).

\begin{figure*}
\begin{center}
\begin{tikzpicture}

%Base nodes
\node[draw, circle, minimum width=1cm] (S1) at (0,0) {$s_{n}$};
\node[draw, circle, minimum width=1cm] (S2) at (12,0) {$s_{n+1}$};
\node[draw, minimum width=0.5cm, minimum height=0.5cm] (C) at (6,0) {$c_{n}$};

%Intermediate nodes in first half
\node[draw, circle] (P1) at (2,2) {};
\node[draw, circle] (P2) at (2,0.67) {};
\node[draw, circle] (P3) at (2,-0.67) {};
\node[draw, circle] (P4) at (2,-2) {};

\node[draw, circle] (P5) at (4,2) {};
\node[draw, circle] (P6) at (4,0.67) {};
\node[draw, circle] (P7) at (4,-0.67) {};
\node[draw, circle] (P8) at (4,-2) {};

%Intermediate nodes in second half
\node[draw, circle] (Q1) at (9,2) {};
\node[draw, circle] (Q2) at (9,0.67) {};
\node[draw, circle] (Q3) at (9,-0.67) {};
\node[draw, circle] (Q4) at (9,-2) {};

\node[draw, circle, minimum width=0.5cm] (Dead) at (7,-3) {$\perp$};

%Invisible nodes
\node[] (I1) at (2,1.8) {};
\node[] (I2) at (2,0.87) {};
\node[] (I3) at (2,0.47) {};
\node[] (I4) at (2,-0.47) {};

\node[] (I5) at (4,1.8) {};
\node[] (I6) at (4,0.87) {};
\node[] (I7) at (4,0.47) {};
\node[] (I8) at (4,-0.47) {};

\node[] (I9) at (9,1.8) {};
\node[] (I10) at (9,0.87) {};
\node[] (I11) at (9,0.47) {};
\node[] (I12) at (9,-0.47) {};

%Dotted paths in first half
\draw[dotted, thick] (P1) -- (P5);
\draw[dotted, thick] (P2) -- (P6);
\draw[dotted, thick] (P3) -- (P7);
\draw[dotted, thick] (P4) -- (P8);

\draw[dotted, thick] (I1) -- (I2);
\draw[dotted, thick] (I3) -- (I4);
\draw[dotted, thick] (I5) -- (I6);
\draw[dotted, thick] (I7) -- (I8);

%Paths in first half
\draw[->,>=latex] 
(S1) edge node[above, midway, sloped]{\scriptsize$\delta_{k(n)}(n)$} (P1)
(S1) edge node[above, midway, sloped]{\scriptsize$\delta_{i}(n)$} (P2)
(S1) edge node[above, midway, sloped]{\scriptsize$\delta_{1}(n)$} (P3)
(S1) edge node[below, midway, sloped]{\scriptsize$\delta_{0}(n)$} (P4);

\draw[->,>=latex] 
(P5) edge (C)
(P6) edge (C)
(P7) edge (C)
(P8) edge (C);

%Overbraces
\draw [decorate, decoration={brace,amplitude=8pt},xshift=0pt,yshift=0pt]
(2.2,2) -- (3.8,2) node [black,midway,above, yshift=5pt] {\scriptsize $n m_{n}^{k(n)}$ steps};
\draw [decorate, decoration={brace,amplitude=8pt},xshift=0pt,yshift=0pt]
(2.2,0.67) -- (3.8,0.67) node [black,midway,above, yshift=5pt] {\scriptsize $nm_{n}^{i}$ steps};
\draw [decorate, decoration={brace,amplitude=8pt},xshift=0pt,yshift=0pt]
(2.2,-0.67) -- (3.8,-0.67) node [black,midway,above, yshift=5pt] {\scriptsize $nm_{n}$ steps};
\draw [decorate, decoration={brace,amplitude=8pt},xshift=0pt,yshift=0pt]
(2.2,-2) -- (3.8,-2) node [black,midway,above, yshift=5pt] {\scriptsize $n$ steps};

%Paths in second half
\draw[->,>=latex] 
(C) edge node [midway, above, sloped]{\scriptsize $-m_{n}^{k(n)}$} (Q1)
(C) edge node [midway, above, sloped]{\scriptsize $-m_{n}^{i}$} (Q2)
(C) edge node [midway, above, sloped]{\scriptsize $-m_{n}$} (Q3)
(C) edge node [midway, above, sloped, pos=0.4]{\scriptsize $-1$} (Q4);

\draw[->,>=latex] 
(Q1) edge node [midway, above, sloped]{\scriptsize $+m_{n}^{k(n)}$} (S2)
(Q2) edge node [midway, above, sloped]{\scriptsize $+m_{n}^{i}$} (S2)
(Q3) edge node [midway, above, sloped]{\scriptsize $+m_{n}$} (S2)
(Q4) edge node [midway, above, sloped, pos=0.6]{\scriptsize $+1$} (S2);

\draw[->,>=latex] 
(Q2) edge node [near end, left]{\scriptsize $\eps_{i}(n)$} (Dead)
(Q3) edge node [pos=0.65, right]{\scriptsize $\eps_{1}(n)$} (Dead)
(Q4) edge node [midway, below]{\scriptsize $\eps_{0}(n)$} (Dead);

%Dotted paths in second half
\draw[dotted, thick] (I9) -- (I10);
\draw[dotted, thick] (I11) -- (I12);

\end{tikzpicture}
\caption{All transition rewards are $0$ unless specified. Recall that $\sum \delta_{i}(n) \cdot \eps_{i}(n)$ is convergent and $\sum \delta_{j}(n) \cdot \eps_{i}(n)$ is divergent for all $i,j$ with $j > i$. The negative reward incurred before falling into the $\perp$ state is reimbursed. We do not show it in the figure for readability.
In the state before $s_{n+1}$, if the correct transition was chosen, the mean payoff is $-1/n$. If the incorrect transition was chosen, then either the mean payoff is $<-m_{n}/n$, or the risk of falling into $\perp$ is too high.
}
\label{stepcounter}
\end{center}
\end{figure*}
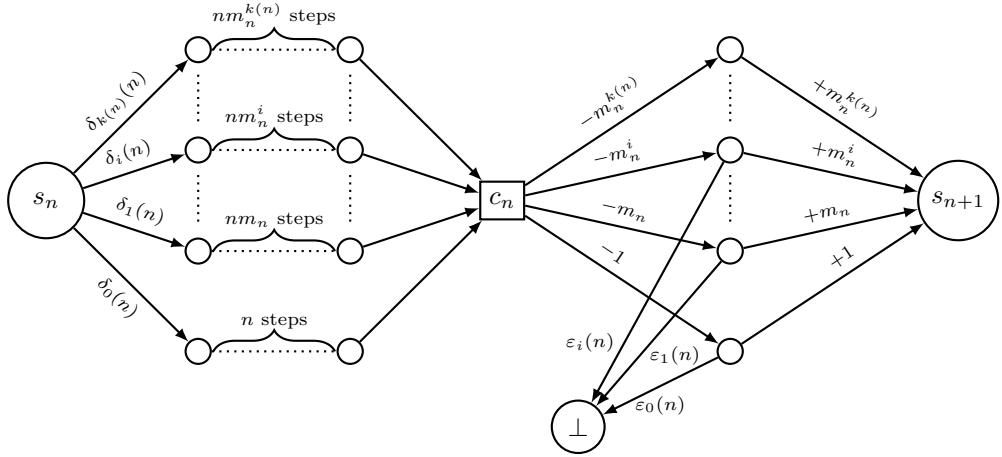

We use the example from \cref{stepcounter}. It is very similar to \cref{infinitegadget}, but differs in the following ways. 
\begin{itemize}
\item The current total reward level is implicit in each state. 
\item The step counter is no longer implicit in the state.
\item In the random choice, instead of changing the reward levels in each choice, it is the path length that differs.
\item The definition of $m_{n}$ is different, it is now $m_{n} \defeq \sum_{i = N^{*}}^{n-1} m_{i}^{k(n)}$ with $m_{N^{*}} \defeq 1$.
\end{itemize}

We construct a finitely branching acyclic MDP $\mdp_{\text{RI}}$ (Reward Implicit) which has the total reward implicit in the state. We do so by chaining together the gadgets from \cref{stepcounter} as is shown in \cref{chain}.

\thmmpstepepslower*
\begin{proof}
This follows from \cref{liminfmpstepval1} and \cref{liminfmpstepval0}.
\end{proof}

\begin{lemma}\label{liminfmpstepval1}
 $\valueof{\mdp_{\text{\emph{RI}}}, \liminfmpobj}{(s_{0},0)} = 1$.
\end{lemma}

\begin{proof}
We define a strategy $\sigma$ which, in $c_{n}$ always mimics the random choice in $s_{n}$. 
Playing according to $\sigma$, the only way to lose is by dropping into the bottom state. This is because by mimicking, the mean payoff in each gadget is lower bounded by $-1/n$.
The rest of the proof is identical to \cref{infwin}.
\end{proof}

\begin{lemma}\label{liminfmpstepval0}
Any FR strategy $\sigma$ in $\mdp_{\text{\emph{RI}}}$ is such that $\probm_{\mdp_{\text{\emph{RI}}}, s_{0}, \sigma}(\liminfmpobj)=0$.
\end{lemma}

\begin{proof}
When playing with finitely many memory modes, there are two ways for a run in $\mdp_{\text{RI}}$ to lose. Either it falls into a losing sink, or it never falls into a sink but its mean payoff is $<-1$. The proof that either of these occurs with probability $1$ is the same as in 
\cref{inflose}.
\end{proof}

Now we construct the MDP $\mdp_{\text{Restart}}$ by chaining together the gadgets from \cref{stepcounter} in the way shown in \cref{restart}.

\thmmpstepoptlower*
\begin{proof}
This follows from \cref{liminfmpstepam1} and \cref{liminfmpstepam0}.
\end{proof}

\begin{lemma}\label{liminfmpstepam1}
There exists an HD strategy $\sigma$ such that $\probm_{\mdp_{\text{\emph{Restart}}}, s_{0}, \sigma}(\liminfmpobj)=1$.
\end{lemma}
\begin{proof}
The proof is identical to that of 
\cref{almostwin}.
\end{proof}

\begin{lemma}\label{liminfmpstepam0}
For any FR strategy $\sigma$, $\probm_{\mdp_{\text{\emph{Restart}}}, s_{0}, \sigma}(\liminfmpobj)=0$.
\end{lemma}
\begin{proof}
The proof is identical to that of
 \cref{almostlose}.
\end{proof}

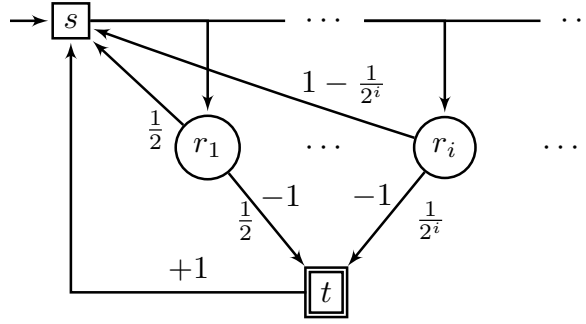
\begin{figure}
\begin{center}
\scalebox{1.2}{
\begin{tikzpicture}[>=latex',shorten >=1pt,node distance=1.9cm,on grid,auto,
roundnode/.style={circle, draw,minimum size=1.5mm},
squarenode/.style={rectangle, draw,minimum size=2mm},
diamonddnode/.style={diamond, draw,minimum size=2mm}]

\node [squarenode,initial,initial text={}] (s) at(0,0) [draw]{$s$};

\node[roundnode] (r1)  [below right=1.4cm and 1.5cm of s] {$r_1$};
%\node[roundnode] (r2)  [right=1cm of r1] {$r_2$};
\node[roundnode] (r3)  [draw=none,right=1.3 of r1] {$\cdots$};
\node[roundnode] (r4)  [right=1.3cm of r3] {$r_i$};
\node[roundnode] (r5)  [draw=none,right=1.3cm of r4] {$\cdots$};

\node [squarenode,double,inner sep = 4pt] (t)  [below=1.6cm of r3] {$ t $};

\draw [->] (s) -- ++(1.5,0) -- (r1);
%\draw [->] (s) -- ++(2,0) -- (r2);
\node[roundnode] (d)  [draw=none,right=2.8 of s] {$\cdots$};
\draw[-] (s) -- (d);
\draw [->] (d) -- ++(1.3,0) -- (r4);
\node[roundnode] (dd)  [draw=none,right=2.8 of d] {$\cdots$};
\draw[-] (d) -- (dd);

\path[->] (r1) edge node [midway,left] {$\frac{1}{2}$} node[near start, right] {$-1$} (t);
%\path[->] (r2) edge node [near start,right=0cm] {\scriptsize{$\frac{1}{4}$}} (t);
\path[->] (r4) edge node [midway,right=.2cm] {$\frac{1}{2^{i}}$} node[near start, left] {$-1$} (t);
\path[->] (r1) edge  node[pos=0.3,below] {$\frac{1}{2}$} (s);   
%\path[->] (r2) edge [bend left=5] node [pos=0.04,above] {\scriptsize{$\frac{3}{4}$}} (s);  
\path[->] (r4) edge [bend left=1] node [pos=0.2,above] {$1-\frac{1}{2^{i}}$} (s);  

\draw [->] (t) -- node[midway, above] {$+1$} ++(-2.8,0) -- (s);
\end{tikzpicture}
}
\caption{ 
%Note that in particular this means that co-B\"{u}chi and lim inf Total Payoff objectives coincide in this MDP when $\{t\}$ is the target set.
We present an infinitely branching MDP adapted from \cite[Figure 3]{KMSW2017} and augmented with a reward structure.
All of the edges carry reward $0$ except the edges entering $t$ that carry reward $-1$ and the edge from $t$ to $s$ carries reward $+1$.
As a result, entering $t$ necessarily brings the total reward down to $-1$ before resetting it to $0$.
We use a reduction to co-B\"{u}chi to show that infinite memory is required for almost-sure as well as $\eps$-optimal strategies for $\liminftpobj$ as well as $\liminfppobj$.
}
\label{infinitebranchtp}
\end{center}
\end{figure}

\thminfbranchsteplower*
\begin{proof}
This follows directly from \cite[Theorem 4]{KMSW2017} and the observation that in \cref{infinitebranchtp}, $\liminftpobj$, $\liminfppobj$ and co-B\"{u}chi objectives coincide.
\end{proof}

Consequently, when the MDP $\mdp$ is infinitely branching and has the reward counter implicit in the state, both $\liminftpobj$ and $\liminfppobj$ require at least a step counter.

% \begin{remark}
% The MDPs depicted in \cref{infinitegadget} and \cref{stepcounter} (when
% chained together appropriately) show that $\liminfmpobj$ requires at least a
% reward counter and a step counter respectively. There does of course exist
% a single MDP which requires at least both a step counter and a reward
% counter. We construct such an MDP by `gluing' the two different MDPs together
% via an initial random state which points to each with probability $1/2$.
% As such, $\eps$-optimal strategies would require both a step counter
% and a reward counter in order to do well, since a single strategy has to
% cope with the constraints of both MDPs.
% \end{remark}

\newpage
\section{Missing proofs from Section~\ref{sec:upper}}\label{app:upper}
\defencodereward

\lemmatotaltopoint*
\lemmatotaltopointproof

\defencodeam

%%%%%%%%%%%%%%%%%%%%%%%%%%%%%%%%%%%%%%%%%%%%%%%%%%%%%%%%%%%%%%%%%%%%%%%%%%%%%%%%%

\subsection{Proofs from Section~\ref{subsec:upper-fb}}\label{app:upper-fb}

In this section we consider finitely branching MDPs.
We need the following technical lemma that holds only for finitely
branching MDPs.

\lemfbavoidstatement
\lemfbavoidproof

\claimfblosetwoeps*
\begin{proof}
%Very similar to the proof of \cref{claim:lose2eps} with $S(\mdp)$ replaced by
% $\mdp'$ and $s_0'$ by $s_0$.
\begin{align*}
& \mathcal{P}_{\mdp',\state_0,\sigma} \left( \bigcap_{i \in \N} \text{Safety}^{n_{i}}_{i} \right) 
\\
& \geq
\mathcal{P}_{\mdp',\state_0,\sigma} \left(
\bigcap_{k \in \N}  \eventually(\text{Safety}_{k}) \cap \bigcap_{i \in \N} \text{Safety}^{n_{i}}_{i} 
\right) \\
& =  \mathcal{P}_{\mdp',\state_0,\sigma} \left( \left( \bigcap_{k \in \N}  \eventually(\text{Safety}_{k}) \cap \bigcap_{i \in \N} \text{Safety}^{n_{i}}_{i} \right)
\cup \left(\overline{\bigcap_{k \in \N}  \eventually(\text{Safety}_{k})} \cap \bigcap_{k \in \N}  \eventually(\text{Safety}_{k})\right)
\right) \\
& = \mathcal{P}_{\mdp',\state_0,\sigma} \left(  \bigcap_{k \in \N}  \eventually(\text{Safety}_{k}) \cap 
\left(\bigcap_{i \in \N} \text{Safety}^{n_{i}}_{i} \cup \overline{\bigcap_{k \in \N}  \eventually(\text{Safety}_{k})} \right)
\right) \\
& = 1 - \mathcal{P}_{\mdp',\state_0,\sigma} \left(  \overline{ \bigcap_{k \in \N}  \eventually(\text{Safety}_{k})} \cup 
\left(\overline{\bigcap_{i \in \N} \text{Safety}^{n_{i}}_{i}} \cap \bigcap_{k \in \N}  \eventually(\text{Safety}_{k})\right)
\right) \\
& \geq 1 - \mathcal{P}_{\mdp',\state_0,\sigma} \left(  \overline{ \bigcap_{k \in \N}  \eventually(\text{Safety}_{k})} \right)
 -
\mathcal{P}_{\mdp',\state_0,\sigma} \left(\overline{\bigcap_{i \in \N} \text{Safety}^{n_{i}}_{i}} \cap \bigcap_{k \in \N}  \eventually(\text{Safety}_{k})\right) \\
& = \mathcal{P}_{\mdp',\state_0,\sigma} \left(\liminfppobj\right)
 -
\mathcal{P}_{\mdp',\state_0,\sigma} \left(\overline{\bigcap_{i \in \N} \text{Safety}^{n_{i}}_{i}} \cap \bigcap_{k \in \N}  \eventually(\text{Safety}_{k})\right)
& \ \text{by \eqref{eq:fbliminfppissafety}}\\
& \geq \valueof{\mdp',\liminfppobj}{s_0} - \eps -
\mathcal{P}_{\mdp',\state_0,\sigma} \left(\bigcup_{i \in \N} \overline{\text{Safety}^{n_{i}}_{i}} \cap \bigcap_{k \in \N}  \eventually(\text{Safety}_{k})\right)
& \ \text{by \eqref{eq:fbsigma-eps-opt}}\\
& \geq \valueof{\mdp',\liminfppobj}{s_0} - \eps
 - \sum_{i \in \N} \mathcal{P}_{\mdp',\state_0,\sigma} \left( \overline{\text{Safety}^{n_{i}}_{i}} \cap \bigcap_{k \in \N} \eventually(\text{Safety}_{k})
\right) \\
& \geq \valueof{\mdp',\liminfppobj}{s_0} - \eps - \sum_{i \in \N} \eps_{i}
& \ \text{by \eqref{eq:fbni}}\\
& = \valueof{\mdp',\liminfppobj}{s_0} - 2 \eps
\end{align*}
\end{proof}

\eqliminfpptransience*
\begin{proof}
First we show that 
\begin{equation}\label{eq:transience-part-liminf}
\transience \subseteq \liminfppobj \quad \text{in $\mdp''$}.
\end{equation}
Let $\rho \in \transience$ be a transient run. Then $\rho$ can never visit the
state $\perp$. Moreover, $\rho$ must eventually leave every finite set
forever. In particular $\rho$ must satisfy $\eventually \always (\neg \text{Bubble}_{n_i}(s_0))$
for every $i$, since $\text{Bubble}_{n_i}(s_0)$ is finite, because $\mdp''$ is
finitely branching. Thus $\rho$ must either fall into $G_{\text{safe}}$,
in which case it satisfies $\liminfppobj$, or for every $i$,
$\rho$ must eventually leave $\text{Bubble}_{n_i}(s_0)$ forever.
By definition of $\text{Bubble}_{n_i}(s_0)$ and $\mdp''$, the run
$\rho$ must eventually stop seeing rewards $< -2^{-i}$ for every $i$.
In this case $\rho$ also satisfies $\liminfppobj$. Thus \eqref{eq:transience-part-liminf}.

Secondly, we show that
\begin{equation}\label{eq:liminf-part-transience}
\forall \sigma''.\ \probm_{\mdp'',s_0,\sigma''}(\liminfppobj \cap \overline{\transience})=0.
\end{equation}
i.e., except for a null-set, $\liminfppobj$ implies $\transience$ in $\mdp''$.

Let $\sigma''$ be an arbitrary strategy from $s_0$ in $\mdp''$ and
$\playset$ be the set of all runs induced by it.
For every $s \in S$, let $\playset_{s} \defeq \{\rho \in \playset \mid \rho \text{ satisfies } \always \eventually (s) \}$
be the set of runs seeing state $s$ infinitely often.
In particular, any run $\rho \in \playset_{s}$ is not transient.
Indeed, $\overline{\transience} = \bigcup_{s \in S} \playset_{s}$.
We want to show that for every state $s \in S$ and strategy $\sigma''$
\begin{equation}\label{eq:liminf-part-transience-s}
\probm_{\mdp'',s_0,\sigma''}(\liminfppobj \cap \playset_{s}) = 0.
\end{equation}
Since any runs seeing a state in $G_{\text{safe}}$ are transient, any
$\playset_{s}$ with $s \in G_{\text{safe}}$ must be empty. Similarly, any run
seeing $\perp$ is losing for $\liminfppobj$ by construction.
Hence we have \eqref{eq:liminf-part-transience-s}
for any state $s$ where $s = \perp$ or $s \in G_{\text{safe}}$.

Now consider $\playset_{s}$ where $s$ is neither in $G_{\text{safe}}$ nor
$\perp$.
Let $T_{\it neg} \eqdef \{t \in \longrightarrow\ \mid\ r(t) < 0\}$
be the subset of transitions with negative rewards in $\mdp''$.

We now show that $\valueof{\mdp'',\neg\eventually T_{\it neg}}{s} < 1$
by assuming the opposite and deriving a contradiction.
Assume that $\valueof{\mdp'',\neg\eventually T_{\it neg}}{s} = 1$.
The objective $\neg\eventually T_{\it neg}$ is a safety objective.
Thus, since $\mdp''$ is finitely branching, there exists a strategy
from $s$ that surely avoids $T_{\it neg}$ (always pick an optimal
move) \cite{Puterman:book,KMSW2017}.
(This does not hold in infinitely branching MDPs where optimal moves might
not exist.)
However, by construction of $\mdp''$, this implies that 
$s \in G_{\text{safe}}$. Contradiction.
Thus $\valueof{\mdp'',\neg\eventually T_{\it neg}}{s} < 1$.

Since $\mdp''$ is finitely branching, we can apply \cref{lem:fbavoid}
and obtain that there exists a threshold $k_s$ such that
$\valueof{\mdp'',\neg\eventually^{\le k_s} T_{\it neg}}{s} < 1$.
Therefore $\delta_s \eqdef 1 - \valueof{\mdp'',\neg\eventually^{\le k_s}
T_{\it neg}}{s} >0$.
Thus, under every strategy, upon visiting $s$ there is a chance $\ge \delta_s$
of seeing a transition in $T_{\it neg}$ within the next $\le k_s$ steps.
Moreover, the subset $T^s_{\it neg} \subseteq T_{\it neg}$
of transitions that can be reached
in $\le k_s$ steps from $s$ is finite, since $\mdp''$ is finitely branching.
So the maximum of the rewards in $T^s_{\it neg}$ is still negative, i.e.,
$\ell_s \eqdef \max\{r(t)\ \mid\ t \in T^s_{\it neg}\} < 0$.
Let $T_{\le \ell} \eqdef \{t \in \longrightarrow\ \mid\ r(t) \le \ell_s\}$
be the subset of transitions with rewards $\le \ell_s$ in $\mdp''$.

Thus, under \emph{every} strategy, upon visiting $s$ there is a chance $\ge \delta_s$
of seeing a transition in $T_{\le \ell}$ within the next $\le k_s$ steps.

Define $\playset_{s}^{i} \defeq \{ \rho \in \playset \mid \rho \text{ sees $s$
at least $i$ times} \}$, so we get $\playset_{s} = \bigcap_{i \in \N} \playset_{s}^{i}$.
We obtain
\begin{align*}
& \sup_{\sigma''}\probm_{\mdp'',s_0,\sigma''}(\liminfppobj \cap \playset_{s}) \\
& \le \sup_{\sigma''}\probm_{\mdp'',s_0,\sigma''}(\eventually\always\neg
T_{\le \ell} \cap \playset_{s}) & \text{set inclusion}\\
&
= \sup_{\sigma''}\lim_{n \to \infty}\probm_{\mdp'',s_0,\sigma''}(\eventually^{\le
n}\always\neg T_{\le \ell} \cap \playset_{s}) &  \text{continuity of measures}\\
& \le \sup_{\sigma'''}\probm_{\mdp'',s,\sigma'''}(\always\neg
T_{\le \ell} \cap \playset_{s}) & \text{$s$ visited after $>n$ steps}\\
& = \sup_{\sigma'''} \probm_{\mdp'',s,\sigma'''}(\always\neg
T_{\le \ell} \cap \bigcap_{i \in \N} \playset_{s}^{i})
& \text{def.\ of $\playset_{s}^{i}$} \\
& = \sup_{\sigma'''} \lim_{i \to \infty}\probm_{\mdp'',s,\sigma'''}(\always\neg
T_{\le \ell} \cap \playset_{s}^{i})
&  \text{continuity of measures} \\
& \le \lim_{i \to \infty}(1-\delta_s)^i = 0 & \text{by def.\ of $\playset_{s}^{i}$ and $\delta_s$}
\end{align*}
and thus \eqref{eq:liminf-part-transience-s}.

From this we obtain
$\probm_{\mdp'',s_0,\sigma''}(\liminfppobj \cap \overline{\transience})=
\probm_{\mdp'',s_0,\sigma''}(\liminfppobj \cap \bigcup_{s \in S} \playset_{s}) \le
\sum_{s \in S} \probm_{\mdp'',s_0,\sigma''}(\liminfppobj \cap \playset_{s})=0$
and thus \eqref{eq:liminf-part-transience}.

From \eqref{eq:transience-part-liminf}
and \eqref{eq:liminf-part-transience}
we obtain that for every $\sigma''$ we have
\begin{align*}
& \probm_{\mdp'',s_0,\sigma''}(\liminfppobj) \\
& = \probm_{\mdp'',s_0,\sigma''}(\liminfppobj \cap \transience) + \probm_{\mdp'',s_0,\sigma''}(\liminfppobj \cap \overline{\transience}) \\
& = \probm_{\mdp'',s_0,\sigma''}(\transience) + 0 \\
& = \probm_{\mdp'',s_0,\sigma''}(\transience)
\end{align*}

and thus \cref{eqliminfpptransience}.
\end{proof}

%%%%%%%%%%%%%%%%%%%%%%%%%%%%%%%%%%%%%%%%%%%%%%%%%%%%%%%%%%%%%%%%%%%%%%%%%%%%%%%%%%%%%%%%%%%%%%

\subsection{Proofs from Section~\ref{subsec:upper-ib}}\label{app:upper-ib}

In this section we consider infinitely branching MDPs.
In the following theorem we show
how to obtain $\eps$-optimal deterministic Markov strategies
for $\liminfppobj$.
We do this by deriving $\eps$-optimal MD strategies in $S(\mdp)$ via
a reduction to a safety objective.

\thminfpointpayoff*
\thminfpointpayoffproof

\cormpepsupper*
\cormpepsupperproof

\corinftpepsupper*
\corinftpepsupperproof

\begin{remark}
While $\eps$-optimal strategies for mean payoff and total payoff
(in infinitely branching MDPs) have the same memory requirements, the step counter and the reward counter do not arise in the same way.
Both the step counter and reward counter used in $\eps$-optimal strategies for mean payoff  arise from the construction of $A(\mdp)$. However, in the case for total payoff, only the reward counter arises from the construction of $R(\mdp)$. The step counter on the other hand arises from the Markov strategy needed for point payoff in $R(\mdp)$. 
\end{remark}

\corinfoptupper*
\corinfoptupperproof

%%%%%%%%%%%%%%%%%%%%%%%%%%%%%%%%%%%%%%%%%%%%%%%%%%%%%%%%%%%%%%%%%%%%%%%%%%%%%%%%%%%%%%%%%%%%%

\newpage
\section{Strengthening results}\label{app:strengthening}
We show that the counterexamples presented in \cref{sec:liminfreward} can be modified s.t.\ all transition rewards
are either $-1$, $0$, or $+1$ and the maximal degree of branching is $2$.
I.e., the hardness does not depend on arbitrarily large rewards or degrees of
branching.

Consider a new MDP $\mathcal{M}$ based on the MDP constructed in \cref{infinitegadget} which now undergoes the following changes.
The rewards on transitions are now limited to $-1$, $0$ or $1$.
To compensate for the smaller rewards, in the $n$-th gadget, each transition bearing a reward is replaced by $k(n) \cdot m_{n}$ transitions as follows. If the original transition had reward $j \cdot m_{n}$ then that transition is replaced with $j \cdot m_{n}$ transitions with reward $1$, and $(k(n) - j) \cdot m_{n}$ transitions with reward $0$. Symmetrically all negatively weighted transitions are similarly replaced by transitions with rewards $-1$ and $0$.

We further alter $\mdp$ by modifying \cref{infinitegadget} such that the
branching degree is bounded by $2$.
We do this by replacing the outgoing transitions in states $s_{n}$ and $c_{n}$ of each gadget by binary trees with accordingly adjusted probabilities such that there is still a probability of $\delta_{i}(n)$ of receiving reward $i \cdot m_{n}$ in each gadget for $i \in \{0, 1, ..., k(n)\}$.

To adjust for the increased path lengths incurred by the modifications to each gadget, the construction in \cref{chain} is accordingly modified by padding each vertical column of white states with extra transitions based on the number of transitions present in the matching gadget. As a result, path length is preserved even when skipping gadgets.
The construction in \cref{restart} is similarly modified.

This construction allows us to obtain the following properties.

\begin{remark}
There exists a countable, acyclic MDP $\mathcal{M}$, whose step counter is implicit in the state, whose rewards on transitions are in $\{-1, 0, 1\}$ and whose branching degree is bounded by $2$ for which 
$\valueof{\mdp,\liminfmpobj}{s_{0}} = 1$ and any FR strategy $\sigma$ is such that 
$\probm_{\mdp, s_{0}, \sigma}(\liminfmpobj)=0$.
In particular, there are no $\varepsilon$-optimal step counter plus finite memory strategies
for any $\varepsilon < 1$ for the
$\liminfmpobj$ objective for countable MDPs.
\end{remark}

This follows from \cref{infwin}, \cref{inflose} and the above construction.

\begin{remark}
There exists a countable, acyclic MDP $\mathcal{M}$, whose step counter is implicit in the state, whose rewards on transitions are in $\{-1, 0, 1\}$ and whose branching degree is bounded by $2$ for which  
$s_{0}$ is almost surely winning and any FR strategy $\sigma$ is such that 
$\probm_{\mdp, s_{0}, \sigma}(\liminfmpobj)=0$.
In particular, almost sure winning strategies, when they exist, cannot be chosen 
with a step counter plus finite memory for countable MDPs.
\end{remark}

This follows from \cref{almostwin}, \cref{almostlose} and the above construction.

\begin{remark}
The two previous remarks also hold for the $\liminftpobj$
objective with no modifications to their respective constructions or their proofs.
\end{remark}

\begin{remark}
The result from \cref{inflose} holds even for strategies $\sigma$ whose memory grows unboundedly, but slower than $k(n)-1$. That is to say that there exists a countable, acyclic MDP $\mathcal{M}$, whose step counter is implicit in the state such that $\valueof{\mdp,\liminfmpobj}{s_{0}} = 1$ and any strategy $\sigma$ with memory $< k(n) - 1$ is such that
 $\probm_{\mathcal{M}, \sigma, s_{0}}(\liminfmpobj)=0$.
 %This follows from a slightly modified version of \cref{alglose} which considers the situation where states $i(n)$ and $k(n)$ are confused in the player's memory. Then the argument used in \cref{inflose} can be modified to include $i(n), j(n): \mathbb{N} \to \{ 0, 1, ..., k(n)\}$.
The result then follows since in every gadget at least one memory mode will confuse at least two states $i(n), j(n): \mathbb{N} \to \{ 0, 1, ..., k(n)-1\}$.
 \end{remark}
 
\bigskip

The results from \cref{sec:liminfstep} can similarly be strengthened. Consider the construction in \cref{stepcounter}. In the random choice, the transition rewards are already all $+1$, so only the branching degree needs to be adjusted by padding the choice with a binary tree as above. In the controlled choice, the transitions carrying reward $ \pm m_{n}^{i}$ are replaced by $m_{n}^{i}$ transitions each bearing reward $\pm 1$ respectively. 
Therefore, the path lengths increase in the following way in the $n$-th
gadget. In $s_{n}$ and $c_{n}$, the binary trees increase path length by up to
$\lceil\lg(k(n)+1)\rceil$ (where $\lg$ is the logarithm to base $2$)
and after $c_{n}$ the path length increases by up to $m_{n}^{k(n)}$ twice.

Consider the scenario where the play took the $i$-th random choice and the
player makes the `best' mistake where they choose transition $i+1$.
We show that, even in this best error case (and thus in all other error
cases), the newly added path lengths do
still not help to prevent seeing a mean payoff $\le -1/2$ in the $n$-th gadget.
In this case, in the state between $c_{n}$ and $s_{n+1}$, the total payoff is $-m_{n}^{i+1}$ and the total number of steps taken by the play so far is upper bounded by 
\[
\beta_n \eqdef \left( \sum_{i=N^{*}}^{n-1} 2\lceil\lg(k(i)+1)\rceil + 2m_{i}^{k(i)} \right)
+ 2\lceil\lg(k(n)+1)\rceil + m_{n}^{i} + m_{n}^{i+1}.
\]
Recall that $m_{n} \defeq \sum_{i = N^{*}}^{n-1} m_{i}^{k(n)}$ with $m_{N^{*}} \defeq 1$, and this is the definition of $m_{n}$ from \cref{app:step} which is different from the definition of $m_{n}$ in \cref{sec:liminfreward}. Note that $k(n)$ is very slowly growing, so it follows that
\[
\beta_n \leq 3m_{n} + m_{n}^{i} + m_{n}^{i+1} \leq 2m_{n}^{i+1}.
\]
That is to say that the mean payoff is $\le \dfrac{-m_{n}^{i+1}}{2m_{n}^{i+1}}
= -1/2$. As a result, in the case of a bad aggressive decision,
the mean payoff will still drop below $-1/2$
in this modified MDP (instead of dropping below $-1$ in the original MDP).
This is just as good to falsify $\liminfmpobj$.

Thus we obtain the following two results.
 
\begin{remark}
There exists a countable, acyclic MDP $\mathcal{M}$, whose reward counter is implicit in the state, whose rewards on transitions are in $\{-1, 0, 1\}$ and whose branching degree is bounded by $2$ for which 
$\valueof{\mdp,\liminfmpobj}{s_{0}} = 1$ and any FR strategy $\sigma$ is such that 
$\probm_{\mdp, s_{0}, \sigma}(\liminfmpobj)=0$.
In particular, there are no $\varepsilon$-optimal step counter plus finite memory strategies
for any $\varepsilon < 1$ for the
$\liminfmpobj$ objective for countable MDPs.
\end{remark}

This follows from \cref{liminfmpstepval1}, \cref{liminfmpstepval0} and the above construction.

\begin{remark}
There exists a countable, acyclic MDP $\mathcal{M}$, whose reward counter is implicit in the state, whose rewards on transitions are in $\{-1, 0, 1\}$ and whose branching degree is bounded by $2$ for which  
$s_{0}$ is almost surely winning and any FR strategy $\sigma$ is such that 
$\probm_{\mdp, s_{0}, \sigma}(\liminfmpobj)=0$.
In particular, almost sure winning strategies, when they exist, cannot be chosen 
with a step counter plus finite memory for countable MDPs.
\end{remark}

This follows from \cref{liminfmpstepam1}, \cref{liminfmpstepam0} and the above construction.

\end{document}